\documentclass[10pt, conference, compsocconf]{IEEEtran}
\pdfoutput=1
\usepackage{cite}
\usepackage{url}
\usepackage{amsmath,amssymb,amsthm}
\usepackage{graphicx}
\usepackage[ruled]{algorithm2e}
\usepackage{bm}

\newtheorem{theorem}{Theorem}

\newtheorem{definition}[theorem]{Definition}
\newtheorem{corollary}[theorem]{Corollary}
\newtheorem{lemma}[theorem]{Lemma}
\newtheorem{remark}[theorem]{Remark}
\newtheorem{problem}[theorem]{Problem}
\newcommand{\trace}{\mathop{\mathrm{trace}}}
\newcommand{\diag}{\mathop{\mathrm{diag}}}
\newcommand{\Diag}{\mathop{\mathrm{Diag}}}
\newcommand{\rank}{\mathop{\mathrm{rank}}}
\newcommand{\Null}{\mathop{\mathrm{Null}}}
\usepackage{hyperref}
\hypersetup{
    colorlinks=true,%
    citecolor=blue,%
    filecolor=blue,%
    linkcolor=blue,%
    urlcolor=blue
}
\begin{document}
%
\title{Robust Low-Rank Subspace Segmentation with Semidefinite Guarantees}

\author{
Yuzhao Ni$^3$, Ju Sun$^{1, 2}$, Xiaotong Yuan$^2$, Shuicheng Yan$^2$, Loong-Fah Cheong$^2$ \\
\small
$^1$ Interactive \& Digital Media Institute, National University of Singapore, Singapore\\
$^2$ Department of Electrical and Computer Engineering, National University of Singapore, Singapore\\
$^3$ School of Computing, National University of Singapore, Singapore
}

\maketitle

\begin{abstract}
Recently there is a line of research work proposing to employ Spectral Clustering (SC) to segment (group)\footnote{Throughout the paper, we use segmentation, clustering, and grouping, and their verb forms, interchangeably.} high-dimensional structural data such as those (approximately) lying on subspaces\footnote{We follow~\cite{liu2010robust} and use the term ``subspace'' to denote both linear subspaces and affine subspaces. There is a trivial conversion between linear subspaces and affine subspaces as mentioned therein.} or low-dimensional manifolds. By learning the affinity matrix in the form of sparse reconstruction, techniques proposed in this vein often considerably boost the performance in subspace settings where traditional SC can fail. Despite the success, there are fundamental problems that have been left unsolved: the spectrum property of the learned affinity matrix cannot be gauged in advance, and there is often one ugly symmetrization step that post-processes the affinity for SC input. Hence we advocate to enforce the symmetric positive semidefinite constraint explicitly during learning (Low-Rank Representation with Positive SemiDefinite constraint, or LRR-PSD), and show that factually it can be solved in an exquisite scheme efficiently instead of general-purpose SDP solvers that usually scale up poorly. We provide rigorous mathematical derivations to show that, in its canonical form, LRR-PSD is equivalent to the recently proposed Low-Rank Representation (LRR) scheme~\cite{liu2010robust}, and hence offer theoretic and practical insights to both LRR-PSD and LRR, inviting future research. As per the computational cost, our proposal is at most comparable to that of LRR, if not less. We validate our theoretic analysis and optimization scheme by experiments on both synthetic and real data sets.
\end{abstract}


\section{Introduction}

This paper deals with grouping or segmentation of high-dimensional data under subspace settings. The problem is formally defined as follows
\begin{problem}[Subspace Segmentation]
Given a set of sufficiently dense data vectors $\mathbf{X}=\left[\mathbf{x}_1, \cdots, \mathbf{x}_n\right]$, $\mathbf{x}_i\in\mathbb{R}^d$ representing a sample $\forall i=1, \cdots, n$. Suppose the data are drawn from a union of $k$ subspaces $\left\{\mathbf{S}_i\right\}_{i=1}^k$ of unknown dimensions $\left\{d_i\right\}_{i=1}^k$ respectively, segment all data vectors into their respective subspaces.
\end{problem}
In this regard, the vast number of available clustering algorithms, ranging from the most basic k-means method to the most elegant and sophisticated spectral clustering (SC) method, can all be used towards a solution. Nevertheless, there are strong reasons to believe that exploiting the very regularity associated with the data can enhance the clustering performance.

We choose SC as the basic framework for subspace segmentation. SC has been extensively researched (see ~\cite{von2007tutorial} for a recent review) and employed for many applications (e.g. image segmentation~\cite{shi2000normalized} in computer vision). SC has the remarkable capacity to deal with highly complicated data structures which may easily fail simple clustering methods such as k-means. The excellent performance of SC can be partially explained via its connection to the kernel method which has been extensively studied in machine learning, specially the recent unification of weighted kernel k-means and SC~\cite{dhillon2004kernel}. The implicit data transformation into higher-dimensional spaces is likely to make the clustering task easy for basic clustering algorithms.

Analogous to the freedom to choose the kernel function in kernel methods, SC is flexible enough to admit any similarity measures in the form of affinity matrices as its input. Despite the existence of research work on SC with general affinity matrices that are not positive semidefinite (see e.g., \cite{meila2007clustering}), in practice the Gaussian kernel $s\left(\mathbf{x}_i, \mathbf{x}_j\right)=\exp\left(-\Vert \mathbf{x}_i-\mathbf{x}_j\Vert^2/\sigma^2\right)$ and the linear kernel $s\left(\mathbf{x}_i, \mathbf{x}_j\right)=\mathbf{x}^\top\mathbf{x}_j$ are evidently the most commonly employed. Use of these kernels naturally ensures about symmetry and positive semidefiniteness of the affinity matrix. When there are processing steps that cause asymmetry, e.g., construction of nearest-neighbors based affinity matrix, there is normally an additional symmetrization step involved before the subsequent eigen-analysis on the resultant Laplacian matrix in normal SC routines~\cite{von2007tutorial}.

\subsection{Prior Work on Subspace Segmentation}
The most intuitive way to solve the subspace segmentation problem is perhaps by robust model fitting. In this aspect, classic robust estimation methods such as RANSAC~\cite{fischler1981random} and Expectation Minimization~\cite{duda2001pattern} can be employed, based on some assumptions about the data distribution, and possibly also the parametric form (e.g., mixture of Gaussians).

Most customized algorithms on this problem, however, are contributed from researchers in computer vision, to solve the 3D multibody motion segmentation (MMS) problem (see e.g., \cite{tron2007benchmark} for the problem statement and review of existing algorithms). In this problem, geometric argument shows that trajectories of same rigid-body motion lie on a subspace of dimension at most $4$. Hence MMS serves as a typical application of subspace segmentation. There are factorization based methods~\cite{costeira1998multibody}, algebraic methods exemplified by the Generalized Principal Component Analysis (GPCA)~\cite{ma2008estimation}, and local subspace affinity (LSA)~\cite{yan2006general} to address MMS. They are all directly or indirectly linked to SC methods, and can be considered as different ways to construct the affinity matrix for subsequent SC (for the former is similar to the linear kernel\footnote{Wei and Lin~\cite{wei2010sp} have concurrently got similar results as produced in Sec.\ref{sec:overlap} of the current paper, and also they have identified the closed-form solution of LRR with that of the shape interaction matrix (SIM) in the classic factorization method.}, and the latter two kernels defined with local subspace distances).

Of special interest to the current investigation is the recent line of work on constructing the affinity matrix by sparsity-induced optimization. Cheng et al~\cite{cheng2010L1} ($\ell_1$ graph ) and Elhamifar et al~\cite{elhamifar2009sparse} (sparse subspace clustering, SSC) have independently proposed to use sparse reconstruction coefficients as similarity measures. To obtain the sparse coefficients, they reconstruct one sample using all the rest samples, while regularizing the coefficient vector by $\ell_1$ norm to promote sparsity. Hence the problem to solve boils down to the Lasso (i.e., $\ell_1$-regularized least square problem, \cite{tibshirani1996regression}), which has been well studied on theoretic and computational sides (ref e.g., \cite{becker2009nesta}). Most recently Liu et al~\cite{liu2010robust} has proposed to compute the reconstruction collectively, and regularize the rank of the affinity matrix for capturing global data structures. This is made possible by employing the nuclear norm minimization as a surrogate, and they also provide a robust version to resist noise and outliers.

Nuclear norm minimization as a surrogate for rank minimization is a natural generalization of the trace heuristic used for positive semidefinite matrices in several fields, such as control theory~\cite{fazel2002matrix}. The need for rank minimization has theoretically stemmed from the exploding research efforts in compressed sensing sparkled by the seminal paper~\cite{candes2005decoding}. In fact, generalizing from vector sparsity to spectrum sparsity for matrices is natural. The practical driving forces come from applications such as collaborative filtering, sensor localization, to just name a few~\cite{candes2009exact}. From the computational side, there are several cutting-edge customized algorithms for solving the otherwise large-scale SDP problem that is likely to plague most popular off-the-shelf SDP solvers. These algorithms include prominently singular value thresholding~\cite{cai2008singular}, accelerated proximal gradient (APG)~\cite{toh2009accelerated}, and augmented Lagrange multiplier (ALM) methods~\cite{lin2009augmented} (see \cite{lin2009augmented} for a brief review).
%
%

\subsection{Our Investigation and Contributions}
We advocate to learn an affinity matrix that makes a valid kernel directly, i.e., being symmetric positive semidefinite. This is one critical problem the previous sparse-reconstruction and global low-rank minimization approaches has bypassed. Without consideration in this aspect, the empirical behaviors of the learnt affinity matrices are poorly justified. We will focus on the global framework proposed in~\cite{liu2010robust} as the global conditions are easier to gauge and additional constraints on the learnt affinity matrix can be put in directly.

We will constrain the affinity matrix to be symmetric positive semidefinite directly in LRR-PSD. Surprisingly, during analysis of connection with the canonical form of LRR proposed in~\cite{liu2010robust}, we find out the two formulations are exactly equivalent, and moreover we can accurately characterize the spectrum of the optimal solution. In addition, we successfully establish the uniqueness of the solution to both LRR and LRR-PSD, and hence correct one critical error reported in~\cite{liu2010robust} stating that the optimal solutions are not unique. 

More interestingly, we show that our advocated formulation (LRR-PSD) in its robust form also admits a simple solution like that of LRR as reported in~\cite{liu2010robust}, but at a lower computational cost. As a nontrivial byproduct, we also provide a rigorous but elementary proof to nuclear-norm regularized simple least square problem with a positive semidefinite constraint, which complements the elegant closed-form solution to the general form~\cite{cai2008singular}.

%
%
%
%

To sum up, we highlight our contributions in two aspects: 1) we provide a rigorous proof of the equivalence between LRR and LRR-PSD, and establish the uniqueness of the optimal solution. In addition, we offer a sensible characterization of the spectrum for the optimal solution;
and 2) we show that Robust LRR-PSD can also be efficiently solved in a scheme similar to that of LRR but with notable difference at a possibly lower cost.

\section{Robust Low-Rank Subspace Segmentation with Semidefinite Guarantees} \label{sec:analysis}
We will first set forth the notation used throughout the paper, and introduce necessary analytic tools. The canonical optimization framework of learning a low-rank affinity matrix for subspace segmentation/clustering will be presented next, and the equivalence between LRR-PSD and LRR will be formally established. Taking on the analysis, we briefly discuss about the spectrum in the robust versions of LRR-PSD and LRR, and touches on other noise assumptions. We will then proceed to present the optimization algorithm to tackle LRR-PSD under noisy settings (i.e., Robust LRR-PSD).
\subsection{Notation and Preliminaries}
\subsubsection{Summary of Notations}
We will use bold capital and bold lowercase for matrices and vectors respectively, such as $\mathbf{X}$ and $\mathbf{b}$, and use normal letters for scalars and entries of matrices and vectors, e.g., $\lambda$, $X_{ij}$ (the $\left(i, j\right)^{th}$ entry of matrix $\mathbf{X}$). We will consider real vector and matrix spaces exclusively in this paper and use $\mathbb{R}^n$ or $\mathbb{R}^{m\times n}$ alike to denote the real spaces of proper dimensionality or geometry.

We are interested in five norms of a matrix $\mathbf{X}\in \mathbb{R}^{m\times n}$. The first three are functions of singular values $\left\{\sigma_i\right\}$ and they are: 1) the operator norm or induced 2-norm denoted by $\Vert\mathbf{X}\Vert_2$, which is essentially the largest singular value $\sigma_{\max}$; 2) the Frobenius norm, defined as $\Vert \mathbf{X}\Vert_F=\left(\sum_{i=1}^d \sigma_i^2\right)^{1/2}$; and 3) the nuclear norm, or sum of the singular values $\Vert \mathbf{X}\Vert_*=\sum_{i=1}^d \sigma_i$, assuming $d=\min \left(m, n\right)$. The additional two include: 4) the matrix $\ell_1$ norm $\Vert \mathbf{X}\Vert_1$ which generalizes the vector $\ell_1$ norm to the concatenation of matrix columns; and 5) the group norm $\Vert \mathbf{X}\Vert_{2, 1}$, which sums up the $\ell_2$ norms of columns. Besides, the Euclidean inner product between matrices is $\langle \mathbf{X}, \mathbf{Y}\rangle = \trace\left(\mathbf{X}^\top\mathbf{Y}\right)$. This also induces an alternative calculation of the Frobenius norm, $\Vert \mathbf{X}\Vert_F = \sqrt{\trace\left(\mathbf{X}^\top\mathbf{X}\right)}$.

We will denote the spectrum (the set of eigenvalues) of a square matrix by $\bm\lambda\left(\mathbf{N}\right)$, for $\mathbf{N}\in \mathbb{R}^{n\times n}$ (similarly the collection of singular values for a general matrix $\bm\sigma\left(\mathbf{X}\right)$). We denote the set of all $n\times n$ real symmetric matrix by $\mathcal{S}^n$, and the corresponding positive semidefinite cone as $\mathcal{S}_+^n$ and $\mathbf{N}\in \mathcal{S}^n_+ \Longleftrightarrow \bm\lambda\left(\mathbf{N}\right)\geq \mathbf{0}$ and $\mathbf{N}\in \mathcal{S}^n$, in which $\mathbf{N}$ is said to be positive semidefinite and simply designated as $\mathbf{N}\succeq \mathbf{0}$. We reiterate the requirement on symmetry here since conventionally definiteness is not defined for asymmetric matrices.

In addition, we all use $\diag\left(\mathbf{X}\right)$ and $\Diag\left(\mathbf{b}\right)$ to mean taking the diagonal vector of a matrix and reshape a vector into a diagonal matrix, respectively. Other notations such as $\trace\left(\mathbf{X}\right)$, $\rank\left(\mathbf{X}\right)$ manifest themselves literally.
\subsubsection{Nuclear Norm Minimization and Rank Minimization}
We choose to devote this subsection to reviewing more concrete properties about the nuclear norm due to its significance to this paper in particular and to the whole range of work on low-rank minimization in general.
\begin{definition}[Unitarily Invariant Norms]
A matrix norm $\Vert\cdot\Vert$ is unitarily invariant if $\Vert \mathbf{X} \Vert = \Vert \mathbf{UXV} \Vert$ for all matrices $\mathbf{X}$ and all unitary matrices $\mathbf{U}$ and $\mathbf{V}$ (i.e. $\mathbf{U}^{-1}=\mathbf{U}^\top, \mathbf{V}^{-1}=\mathbf{V}^\top$) of compatible dimension.
\end{definition}
Interestingly common encountered unitarily invariant norms are all functions of the singular values, and lie within two general families: 1) Schatten-$p$ norms, arising from applying the $p$-norm to the vector of singular values, $\Vert \mathbf{X} \Vert_{Sp} = \left(\sum_{i=1}^{d} \sigma_i^p\right)^{1/p}$; and 2)Ky-Fan $k$ norms, representing partial ordered sums of largest singular values, $\Vert \mathbf{X} \Vert_{KFk} = \sum_{i=1}^{k} \sigma_i$, assuming $k\leq d$ and $\sigma_1\geq \cdots \geq \sigma_d$ all for $d=\min\left(m, n\right)$. Of our interest here is that the nuclear norm and Frobenius norm are Schatten-$1$ norm and Schatten-$2$ norms, respectively. This fact will be critical for several places of our later argument.

Next we will state one crucial fact about the duality between the nuclear norm and the operator norm. For any norm $\Vert\cdot\Vert$, its dual norm $\Vert\cdot\Vert^\mathcal{D}$ is defined via the variational characterization ~\cite{recht2008necessary}
\begin{equation}
\Vert \mathbf{X}\Vert^{\mathcal{D}} = \sup_{\mathbf{Y}} \;  \left\{\langle \mathbf{Y}, \mathbf{X}\rangle \; \vert \; \Vert \mathbf{Y} \Vert \leq 1\right\},
\end{equation}
where $\Vert \mathbf{Y}\Vert\leq 1$ can always be taken as equality for the supremum to achieve, since the inner product $\langle \mathbf{Y}, \mathbf{X}\rangle$ is homogeneous w.r.t. $\mathbf{Y}$. Then we have a formal statement about the duality
\begin{lemma}[\cite{recht2008guaranteed}, Proposition 2.1]\label{lemma:duality}
The dual norm of the operator norm $\Vert\cdot\Vert_2$ in $\mathbb{R}^{m\times n}$ is the nuclear norm $\Vert \cdot\Vert_*$.
\end{lemma}
In fact, the duality taken together with the characterization of dual norms implies $\langle \mathbf{Y}, \mathbf{X}\rangle \leq \Vert \mathbf{Y}\Vert_2 \Vert \mathbf{X}\Vert_*$, which has been used extensively in the analysis of nuclear norm problems.

Our last piece of review touches on the core of the problem, i.e., how rank minimization problems (NP-Hard) could be (conditionally) solved via nuclear norm minimization formulation which is convex. This myth lies with the concept of convex envelope, which means the tightest convex pointwise approximation to a function (tightest convex lower bound). Formally, for any (possibly nonconvex, e.g., the rank function currently under investigation) function $f: \mathcal{C}\mapsto \mathbb{R}$, where $\mathcal{C}$ denotes a given convex set, the convex envelope of $f$ is the largest convex function $g$ such that $g\left(x\right) \leq f\left(x\right)$ for all $x\in \mathcal{C}$~\cite{fazel2002matrix}. The following lemma relates the rank function to the nuclear norm via convex envelope
\begin{lemma}[\cite{fazel2002matrix}, Theorem 1, pp.54 and Sec.5.1.5 for proof]
The convex envelope of $\rank\left(\mathbf{X}\right)$ on the set $\left\{\mathbf{X}\in \mathbb{R}^{m\times n}\; \vert \; \Vert \mathbf{X}\Vert_2\leq 1\right\}$ is the nuclear norm $\Vert \mathbf{X}\Vert_*$.
\end{lemma}
This lemma justifies the heuristic to use the nuclear norm as a surrogate for the rank. Much of recent work, e.g. low-rank matrix completion~\cite{candes2009exact} and Robust Principal Component Analysis (RPCA)~\cite{candes2009robust}, proves theoretically under mild conditions, the optimization can be exactly equivalent. We will borrow heavily the idea of this surrogate, and build on the theoretical underpinnings to develop our formulation and analysis of LRR-PSD/LRR for subspace segmentation.

\subsection{Subspace Segmentation with Clean Data -- An Amazing Equivalence} \label{sec:overlap}
To tackle the subspace segmentation problem, Liu et al~\cite{liu2010robust} have proposed to learn the affinity matrix for SC via solving the following rank minimization problem
\begin{equation}
\mbox{(LRANK)} \quad \quad  \min. \; \rank\left(\mathbf{Z}\right), \; \text{s.t.} \; \mathbf{X}=\mathbf{XZ}.
\end{equation}
As a convex surrogate, the rank objective is replaced by the nuclear norm, and hence the formulation
\begin{equation}
\mbox{(LRR)} \quad \quad \quad \min. \; \Vert\mathbf{Z}\Vert_*, \; \text{s.t.} \; \mathbf{X}=\mathbf{XZ}.
\end{equation}
Instead, we advocate to solve the problem incorporating the positive semidefinite constraint directly to produce a valid kernel directly as argued above
\begin{equation}
\mbox{(LRR-PSD)} \quad \min. \; \Vert\mathbf{Z}\Vert_*, \; \text{s.t.} \; \mathbf{X}=\mathbf{XZ}, \mathbf{Z}\succeq \mathbf{0}.
\end{equation}
Liu et al~\cite{liu2010robust} has established one important characterization about solution(s) to \textbf{LRR} for $\mathbf{X}\in \mathbb{R}^{d\times n}$ and $\mathbf{Z}\in \mathbb{R}^{n\times n}$, provided \emph{the data have been arranged by their respective groups, i.e., the true segmentation}.
\begin{theorem}[\cite{liu2010robust}, Theorem 3.1]\label{thm:liu}
Assume the data sampling is sufficient such that $n_i>\rank\left(\mathbf{X}_i\right)=d_i$ and the data have been ordered by group. If the subspaces are independent then there exists an optimal solution $\mathbf{Z}^*$ to problem \textbf{LRR} that is block-diagonal:
\begin{equation}\label{eq:opt_z}
\mathbf{Z}^*_{n\times n} =
\begin{bmatrix}
  \mathbf{Z}_1^* & \mathbf{0}      & \mathbf{0} & \mathbf{0} \\
  \mathbf{0}     & \mathbf{Z}_2^*  & \mathbf{0} & \mathbf{0} \\
  \mathbf{0}     & \mathbf{0}      & \ddots     & \mathbf{0} \\
  \mathbf{0}     & \mathbf{0}      & \mathbf{0} & \mathbf{Z}_k^* \\
\end{bmatrix}
\end{equation}
with $\mathbf{Z}_i^*$ being an $n_i\times n_i$ matrix with $\rank\left(\mathbf{Z}_i^*\right)=d$, $\forall i$.
\end{theorem}
This observation is critical to good segmentation since affinity matrices with block diagonal structure (for sorted data as stated) favor perfect segmentation as revealed by theoretic analysis of SC algorithms (e.g., refer to \cite{von2007tutorial} for brief exposition). There are, however, discoveries that are equally important yet to make. We will next state somewhat surprising results that we have derived, complementing Theorem~\ref{thm:liu} and providing critical insights in characterizing the (identical and unique) solution to LRR-PSD and LRR.
\begin{theorem}\label{thm:our_unique}
Optimization problem \textbf{LRR} has a unique minimizer $\mathbf{Z}^*$. Moreover there exists an orthogonal matrix $\mathbf{Q}\in \mathbb{R}^{n\times n}$ such that
\begin{equation}
\mathbf{Q}^\top \mathbf{Z}^{*}\mathbf{Q} =\begin{bmatrix}
\mathbf{I}_r  & \mathbf{0} \\
\mathbf{0}    & \mathbf{0}
\end{bmatrix}
\end{equation}
where $r=\rank\left(\mathbf{X}\right)$, and obviously $\mathbf{Z}^*\succeq \mathbf{0}$.
\end{theorem}
Three important corollaries follow immediately from Theorem~\ref{thm:our_unique}.
\begin{corollary}[LRR-PSD/LRR Equivalence]
The LRR problem and LRR-PSD problem are exactly equivalent, i.e. with identical unique minimizers that are symmetric positive semidefinite.
\end{corollary}
\begin{proof}
$\mathbf{Z}^*$ in Theorem~\ref{thm:our_unique} naturally obeys LRR-PSD.
\end{proof}
\begin{corollary}
Assume the setting in Theorem~\ref{thm:liu}. The optimal solution $\mathbf{Z}^*$ to problem \textbf{LRR} and \textbf{LRR-PSD} are block-diagonal as in Eq.~\eqref{eq:opt_z}.
\end{corollary}
\begin{proof}
Follow directly from LRR-PSD/LRR equivalence and Theorem~\ref{thm:liu}.
\end{proof}
\begin{corollary}[LRR-PSD/LRR/LRANK Equivalence]
The optimal rank of $\mathbf{Z}$ in \textbf{LRANK} is the objective value obtained from \textbf{LRR-PSD/LRR}, i.e., $\rank\left(\mathbf{X}\right)$.
\end{corollary}
\begin{proof}
Proof of Theorem~\ref{thm:our_unique} (later) shows $\rank\left(\mathbf{Z}\right)$ cannot be lower than $\rank\left(\mathbf{X}\right)$ due to the constraint $\mathbf{X}=\mathbf{XZ}$. $\rank\left(\mathbf{X}\right)$ is the optimal objective value for nuclear norm of $\mathbf{Z}$ since $\Vert \mathbf{Z}^*\Vert_*=\Vert \mathbf{Q}^{\top}\mathbf{Z}^*\mathbf{Q}\Vert_*=\rank\left(\mathbf{X}\right)$.
\end{proof}

The development of the results in Theorem~\ref{thm:our_unique} will testify the beautiful interplay between classic matrix computation and properties of nuclear norms we reviewed above. We present and validate several critical technical results preceding formal presentation of our proof.
\begin{lemma}[\cite{golub1996matrix}, Lemma 7.1.2 on Real Matrices] \label{lemma:schur-lemma}
If $\mathbf{A}\in \mathbb{R}^{n\times n}$, $\mathbf{B}\in \mathbb{R}^{p\times p}$, and $\mathbf{M}\in \mathbb{R}^{n\times p}$ satisfy
\begin{equation} \label{eq:ammb}
\mathbf{AM = MB}, \; \; \rank \left(\mathbf{M}\right) = p,
\end{equation}
then there exists an orthogonal $\mathbf{Q} \in \mathbb{R}^{n\times n}$ such that
\begin{equation}
\mathbf{Q}^\top \mathbf{AQ} = \mathbf{T} =
\begin{bmatrix}
\mathbf{T}_{11} & \mathbf{T}_{12} \\
\mathbf{0} & \mathbf{T}_{22}
\end{bmatrix}
\end{equation}
for $\mathbf{T}\in\mathbb{R}^{n \times n}$, $\mathbf{T}_{11}\in \mathbb{R}^{p\times p}$ and $\mathbf{T}_{12}$, $\mathbf{0}$ and $\mathbf{T}_{22}$ of compatible dimensions. Furthermore,  $\bm\lambda\left(\mathbf{T}_{11}\right) = \bm\lambda\left(\mathbf{A}\right) \cap \bm\lambda \left(\mathbf{B}\right)$.\footnote{We follow the convention in Golub and van Loan~\cite{golub1996matrix} and use $\bm\lambda\left(\cdot\right)$ to denote the set of eigenvalues counting multiplicity. Hence it is not a normal set, and use of set operators here abuses their traditional definitions.}
\end{lemma}
Since the proof is critical for subsequent arguments, we reproduce the sketch here for completeness.
\begin{proof}
Let
\begin{equation}
\mathbf{M} = \mathbf{Q}
\begin{bmatrix}
\mathbf{R}_1 \\
\mathbf{0}
\end{bmatrix}, \mathbf{Q}\in \mathbb{R}^{n\times n}, \mathbf{R}_1\in \mathbb{R}^{p \times p}
\end{equation}
be a QR factorization\footnote{Note that QR factorization may not be unique, complement to the freedom in choosing a basis for $\Null\left(\mathbf{M}^\top\right)$,  which is dual to the column space of $\mathbf{M}$.} of $\mathbf{M}$. By substituting this into Eq.~\eqref{eq:ammb} and rearranging we arrive at
\begin{equation}
\begin{bmatrix}
\mathbf{T}_{11}  & \mathbf{T}_{12} \\
\mathbf{T}_{21}  & \mathbf{T}_{22}
\end{bmatrix}
\begin{bmatrix}
\mathbf{R}_1  \\
\mathbf{0}
\end{bmatrix}
= \begin{bmatrix}
\mathbf{R}_1  \\
\mathbf{0}
\end{bmatrix}\mathbf{B}, \; \text{where} \;
\mathbf{Q}^\top\mathbf{A}\mathbf{Q} =
\begin{bmatrix}
\mathbf{T}_{11}  & \mathbf{T}_{12} \\
\mathbf{T}_{21}  & \mathbf{T}_{22}
\end{bmatrix},
\end{equation}
with $\mathbf{T}_{11}\in \mathbb{R}^{p\times p}$ and others of compatible dimension. Since $\mathbf{R}_1$ is nonsingular, $\mathbf{T}_{21}\mathbf{R}_1 = \mathbf{0}$ implying $\mathbf{T}_{21} = \mathbf{0}$. Moreover, $\mathbf{T}_{11}\mathbf{R}_1=\mathbf{R}_1\mathbf{B}\Leftrightarrow \mathbf{T}_{11}=\mathbf{R}_1\mathbf{B}\mathbf{R}_1^{-1}$, suggesting $\mathbf{T}_{11}$ and $\mathbf{B}$ are similar and hence $\bm\lambda\left(\mathbf{T}_{11}\right)=\bm\lambda\left(\mathbf{B}\right)$. Lemma 7.1.1~\cite{golub1996matrix} dictates that $\bm\lambda\left(\mathbf{A}\right)=\bm\lambda\left(\mathbf{T}\right)
=\bm\lambda\left(\mathbf{T}_{11}\right)\cup\bm\lambda\left(\mathbf{T}_{22}\right)$, which leads to the conclusion.
\end{proof}
The next lemma deals with an important inequality of nuclear norms on vertically-partitioned or horizontally-partitioned matrices.
\begin{lemma}[\cite{stewart1990matrix}, Adaptation of Theorem 4.4, pp 33-34]\label{lemma:norm-red}
Let $\mathbf{A}\in \mathbb{R}^{m\times n}$ be partitioned in the form
\begin{equation}
\mathbf{A} =
\begin{bmatrix}
\mathbf{A}_1 \\
\mathbf{A}_2
\end{bmatrix} \; (\text{resp.} \;
\mathbf{A} =
\begin{bmatrix}
\mathbf{A}_1 &
\mathbf{A}_2
\end{bmatrix})
\end{equation}
and the sorted singular values of $\mathbf{A}$ be $\sigma_1 \geq \sigma_2 \geq \cdots \geq \sigma_d \geq 0$ and those of $\mathbf{A}_1$ be $\tau_1 \geq \tau_2 \geq \cdots \tau_d \geq 0$ for $d=\min \left(m, n\right)$. Then $\Vert \mathbf{A}\Vert_* \geq \Vert \mathbf{A}_1\Vert_*$, where the equality holds if and only if $\mathbf{A}_2=\mathbf{0}$.
\end{lemma}
\begin{proof}
The original proof in \cite{stewart1990matrix} has shown that $\sigma_i \geq \tau_i$, $\forall i=1, \cdots, d$. This is because $\sigma_i^2$ and $\tau_i^2$, $\forall i=1, \cdots, d$ are eigenvalues of $\mathbf{A}^\top\mathbf{A}$ (resp. $\mathbf{A}\mathbf{A}^\top$) and $\mathbf{A}_1^\top \mathbf{A}_1$ (resp. $\mathbf{A}_1\mathbf{A}_1^\top$), respectively, and $\mathbf{A}^\top\mathbf{A} = \mathbf{A}_1^\top\mathbf{A}_1 + \mathbf{A}_2^\top\mathbf{A}_2$ (resp. $\mathbf{A}\mathbf{A}^\top = \mathbf{A}_1\mathbf{A}_1^\top + \mathbf{A}_2\mathbf{A}_2^\top$). Theorem 3.14~\cite{stewart1990matrix} dictates that $\sigma_i^2 \geq \tau_i^2 +\lambda_d^2$, $\forall i=1, \cdots, d$, where $\lambda_d^2$ is the smallest eigenvalue of $\mathbf{A}_2^\top \mathbf{A}_2$ (resp. $\mathbf{A}_2\mathbf{A}_2^\top$). It follows $\sum_{i=1}^d \sigma_i \geq \sum_{i=1}^d \tau_i$, and hence we have $\Vert \mathbf{A}\Vert_* \geq \Vert \mathbf{A}_1\Vert_*$.

We next show stronger results saying that the inequality is strict unless $\mathbf{A}_2 = \mathbf{0}$.

($\Longrightarrow$)
Since $\sigma_i \geq \tau_i$, $\forall i=1, \cdots, d$,  requiring $\Vert \mathbf{A}\Vert_* =\Vert \mathbf{A}_1\Vert_*$ or $\sum_{i=1}^d \sigma_i = \sum_{i=1}^d \tau_i$ amounts to imposing $\sigma_i=\tau_i$, $\forall i=1, \cdots, d$, which suggests
$\sum_{i=1}^n \sigma_i^2 = \sum_{i=1}^n \tau_i^2$, or $\trace\left(\mathbf{A}^\top\mathbf{A}\right) = \trace\left(\mathbf{A}_1^\top\mathbf{A}_1\right)$ (resp. $\trace\left(\mathbf{A}\mathbf{A}^\top\right) = \trace\left(\mathbf{A}_1\mathbf{A}_1^\top\right)$), identically $\Vert \mathbf{A}\Vert_F^2 =\Vert \mathbf{A}_1\Vert_F^2$. But we also have from the above argument $\trace\left(\mathbf{A}^\top\mathbf{A}\right) = \trace\left(\mathbf{A}_1^\top\mathbf{A}_1\right) + \trace\left(\mathbf{A}_2^\top\mathbf{A}_2\right)$ (resp. $\trace\left(\mathbf{A}\mathbf{A}^\top\right) = \trace\left(\mathbf{A}_1\mathbf{A}_1^\top\right) + \trace\left(\mathbf{A}_2\mathbf{A}_2^\top\right)$), or $\Vert \mathbf{A}\Vert_F^2 =\Vert \mathbf{A}_1\Vert_F^2 + \Vert \mathbf{A}_2\Vert_F^2$. Taking them together we obtain $\Vert \mathbf{A}_2\Vert_F^2 =0$, implying $\mathbf{A}_2 =\mathbf{0}$.

($\Longleftarrow$) Simple substitution verifies the equality and also completes the proof.
\end{proof}


Based on the above two important lemmas, we derived our main results as follows.
\begin{proof}\textbf{(of Theorem~\ref{thm:our_unique})}
We first show the claim about the semidefiniteness of $\mathbf{Z}^*$, and then proceed to prove the uniqueness.

(\textbf{Semidefiniteness of $\mathbf{Z}^*$})
By $\mathbf{XZ}=\mathbf{X}$, we have $\mathbf{Z}^\top\mathbf{X}^\top=\mathbf{X}^\top$. Taking $r$ independent columns from $\mathbf{X}^\top$ (i.e., $r$ independent rows from $\mathbf{X}$) and organize them into a submatrix $\mathbf{M}$ of $\mathbf{X}^\top$, we obtain $\mathbf{Z}^\top\mathbf{M}=\mathbf{MI}$. By Lemma~\ref{lemma:schur-lemma} and its proof, we have a QR factorization of $\mathbf{M}$ and one similarity transform of $\mathbf{Z}^\top$ as, respectively
\begin{equation} \label{main:qr_basis}
\begin{aligned}
 \mathbf{M} & =
\begin{bmatrix}
\mathbf{U} & \mathbf{U}^\bot
\end{bmatrix}
\begin{bmatrix}
\mathbf{R} \\ \mathbf{0}
\end{bmatrix}, \; \; \text{and} \; \; \\
 \mathbf{T} & =
 \begin{bmatrix}
\mathbf{U} &
\mathbf{U}^\bot
\end{bmatrix}^\top
\mathbf{Z}^\top
\begin{bmatrix}
\mathbf{U} & \mathbf{U}^\bot
\end{bmatrix} \\
& =\begin{bmatrix}
\mathbf{T}_{11}  & \mathbf{T}_{12} \\
\mathbf{0}  & \mathbf{T}_{22}
\end{bmatrix}
=\begin{bmatrix}
\mathbf{I}_r  & \mathbf{T}_{12} \\
\mathbf{0}  & \mathbf{T}_{22}
\end{bmatrix},
\end{aligned}
\end{equation}
where $\mathbf{U}^\bot$ spans the complementary dual subspace of $\mathbf{U}$. Moreover, we have obtained $\mathbf{T}_{11}=\mathbf{I}$ because proof of Lemma~\ref{lemma:schur-lemma} suggests $\mathbf{T}_{11}=\mathbf{R}\mathbf{I}\mathbf{R}^{-1}=\mathbf{I}$. The dimension is obviously determined by rank of $\mathbf{X}$, i.e., $r=\rank\left(\mathbf{X}\right)$.

We continue to show that towards minimal $\Vert \mathbf{Z}^\top\Vert_*$, $\mathbf{T}_{12} = \mathbf{T}_{22}=\mathbf{0}$. By the unitary invariance property of nuclear norm, $\min. \; \Vert \mathbf{Z}^\top \Vert_* \Longleftrightarrow \min. \; \Vert \mathbf{T}\Vert_*$. Noticing that
\begin{equation}\label{main:zero_out}
\begin{bmatrix}
\mathbf{T}_{12} \\ \mathbf{T}_{22}
\end{bmatrix}
= \begin{bmatrix}
\mathbf{U} & \mathbf{U}^\bot
\end{bmatrix}^\top \mathbf{Z}^\top \mathbf{U}^\bot
\end{equation}
and $\mathbf{Z}^\top\mathbf{M}=\mathbf{M}$ results in two constraints
\begin{equation}
\mathbf{Z}^\top \mathbf{U}\mathbf{R} = \mathbf{U}\mathbf{R} \; \; \text{and} \; \;
\mathbf{Z}^\top \mathbf{U}^\bot \mathbf{0} = \mathbf{U}^\bot \mathbf{0} =\mathbf{0}.
\end{equation}
Under these constraints, we can always make $\mathbf{Z}^\top \mathbf{U}^\bot = \mathbf{0}$, or effectively $\mathbf{T}_{12} = \mathbf{T}_{22}=\mathbf{0}$, attaining the minimizer in that
\begin{equation}\label{main:opt-cond}
\left\Vert
\begin{bmatrix}
\mathbf{I}_r  & \mathbf{T}_{12} \\
\mathbf{0}  & \mathbf{T}_{22}
\end{bmatrix}
\right\Vert_* \geq
\left\Vert
\begin{bmatrix}
\mathbf{I}_r  \\
\mathbf{0}
\end{bmatrix}
\right\Vert_* = \left\Vert
\mathbf{I}_r
\right\Vert_*=r,
\end{equation}
where the first inequality has followed from Lemma~\ref{lemma:norm-red} and equality is obtained only when $\mathbf{T}_{12} = \mathbf{T}_{22}=\mathbf{0}$. Hence we have shown that $\mathbf{Z}^\top = \mathbf{U}\mathbf{U}^\top=\mathbf{Z} \succeq \mathbf{0}$, as an optimal solution of \textbf{LRR}.

(\textbf{Uniqueness of $\mathbf{Z}^*$}) Suppose a perturbed version $\mathbf{Z}'=\mathbf{Z}^* + \mathbf{H}$ is also a minimizer. So $\mathbf{XZ}'=\mathbf{X}\left(\mathbf{Z}^*+\mathbf{H}\right)=\mathbf{X}=\mathbf{X}\mathbf{Z}^*$, suggesting $\mathbf{X}\mathbf{H}=\mathbf{0}$ or $\mathbf{H}$ is in $\Null\left(\mathbf{X}\right)$ (which complements the row space). We have
\begin{equation}
\begin{aligned}
\mathbf{H}^\top\mathbf{X}^\top  = \mathbf{0}& \Longrightarrow \mathbf{H}^\top
\begin{bmatrix}
\mathbf{U}  & \mathbf{U}^\bot
\end{bmatrix}
\begin{bmatrix}
\mathbf{R}  \\ \mathbf{0}
\end{bmatrix} = \mathbf{0} \\
& \Longrightarrow
\mathbf{H}^\top \mathbf{U}\mathbf{R}=\mathbf{0} \Longrightarrow \mathbf{H}^\top \mathbf{U} =\mathbf{0},
\end{aligned}
\end{equation}
where the last equality holds because $\mathbf{R}$ is nonsingular. If $\mathbf{H}\neq \mathbf{0}$, we have
\begin{equation}
\begin{aligned}
& \begin{bmatrix}
\mathbf{U}  & \mathbf{U}^\bot
\end{bmatrix}^\top \mathbf{Z}'^\top
\begin{bmatrix}
\mathbf{U}  & \mathbf{U}^\bot
\end{bmatrix}\\
& =\begin{bmatrix}
\mathbf{T}'_{11}  & \mathbf{T}'_{12} \\
\mathbf{0}  & \mathbf{T}'_{22}
\end{bmatrix}
 =\begin{bmatrix}
\mathbf{I}_r  & \mathbf{U}^\top \mathbf{H}^\top \mathbf{U}^\bot \\
\mathbf{0}  & \left(\mathbf{U}^\bot\right)^\top \mathbf{H}^\top \mathbf{U}^\bot
\end{bmatrix},
\end{aligned}
\end{equation}
where we have substituted the analytic values of $\mathbf{Z}^*$ and its corresponding $\mathbf{T}_{ij}, \forall i, j = \{1, 2\}$ as discussed above and the fact $\mathbf{H}^\top\mathbf{U}=\mathbf{0}$.  Since $\mathbf{H}^\top\mathbf{U}^\bot\neq \mathbf{0}$ (otherwise together with $\mathbf{H}^\top\mathbf{U}=\mathbf{0}$ we would obtain $\mathbf{H}=\mathbf{0}$), employing the inequality in Lemma~\ref{lemma:norm-red} again we see that
\begin{equation}
\left\Vert
\begin{bmatrix}
\mathbf{I}_r  & \mathbf{U}^\top \mathbf{H}^\top \mathbf{U}^\bot \\
\mathbf{0}  & \left(\mathbf{U}^\bot\right)^\top \mathbf{H}^\top \mathbf{U}^\bot
\end{bmatrix}\right\Vert_*
>
\left\Vert
\begin{bmatrix}
\mathbf{I}_r  & \mathbf{0} \\
\mathbf{0}  & \mathbf{0}
\end{bmatrix}\right\Vert_* =
\left\Vert
\mathbf{I}_r
\right\Vert_* =
\left\Vert
\mathbf{Z}^*
\right\Vert_*.
\end{equation}
In other words, the objective is strictly increased unless $\mathbf{H}^\top =\mathbf{0}$ or $\mathbf{H}=\mathbf{0}$, which establishes the uniqueness.
and also concludes the proof.
\end{proof}
\begin{remark}
Nonuniqueness of $QR$ factorization of $\mathbf{M}$ will not affect the uniqueness of $\mathbf{Z}^*$ as follows. Suppose we choose another basis $\mathbf{V}$ for column space of $\mathbf{M}$, obviously $\mathbf{V}$ and $\mathbf{U}$ must be related by a within-space rotation $\mathbf{R}$, i.e., $\mathbf{V}=\mathbf{UR}$. Hence w.r.t. the new basis we have $\mathbf{Z}^*=\mathbf{V}\mathbf{V}^\top=\mathbf{UR}\mathbf{R}^\top\mathbf{U}=\mathbf{U}\mathbf{U}^\top$, as $\mathbf{R}^\top\mathbf{R}=\mathbf{R}\mathbf{R}^\top =\mathbf{I}$.
\end{remark}
\subsection{Robust Subspace Segmentation with Data Containing Outliers and Noises}\label{sec:robust-set}
To account for noises and outliers, explicit distortion terms can be  introduced into the objective and constraint. Hence we obtain the robust version of \textbf{LRR-PSD} and \textbf{LRR} respectively as follows
\begin{equation} \label{eq:robust-PSD}
\min. \; \Vert \mathbf{Z} \Vert_* + \lambda\Vert \mathbf{E} \Vert_{\ell}, \; \text{s.t.} \; \mathbf{X} = \mathbf{X}\mathbf{Z} + \mathbf{E}, \mathbf{Z}\succeq \mathbf{0},
\end{equation}
\begin{equation}\label{eq:robust-form}
\min. \Vert\mathbf{Z}\Vert_*+\lambda\Vert \mathbf{E}\Vert_{\ell}, \; \text{s.t.} \; \mathbf{XZ}+\mathbf{E} = \mathbf{X}.
\end{equation}
We have used $\Vert\cdot\Vert_{\ell}$ to mean generic norms.  We caution that we cannot in general expect these two versions to be equivalent despite the provable equivalence of \textbf{LRR-PSD} and \textbf{LRR}. Remarkably, the problem has changed much due to the extra variable $\mathbf{E}$. Nevertheless, it is still possible to partially gauge the behaviors of the solutions as follows.

Suppose an optimal $\mathbf{E}^*$ is somehow achieved (i.e., we assume it is fixed), we are then only concerned with
\begin{equation}
\min. \; \Vert \mathbf{Z}\Vert_*  \; \text{s.t.} \; \mathbf{X}\mathbf{Z} +\mathbf{E}^* = \mathbf{X}, \left(\mathbf{Z}\succeq \mathbf{0}\right).
\end{equation}
Since columns of $\mathbf{E}^*$ must be in the column space of $\mathbf{X}$, we assume $\mathbf{E}^*=\mathbf{X}\delta\mathbf{E}$.
Then we obtain from the equality constraint $\mathbf{X}\left(\mathbf{Z}+\delta\mathbf{E}\right)=\mathbf{X}$. By employing similar process in the proof of Theorem~\ref{thm:our_unique}, one can easily verify that
\begin{align}
&
\begin{bmatrix}
\mathbf{U}  & \mathbf{U}^\bot
\end{bmatrix}^\top
 \left(\mathbf{Z}^\top+\delta\mathbf{E}^\top\right) \begin{bmatrix}
\mathbf{U}  & \mathbf{U}^\bot
\end{bmatrix} \nonumber \\
=&
\begin{bmatrix}
\mathbf{T}_{11}^\mathbf{Z^\top}  & \mathbf{T}_{12}^{\mathbf{Z}^\top} \\
\mathbf{C}  & \mathbf{T}_{22}^{\mathbf{Z}^\top}
\end{bmatrix} +
\begin{bmatrix}
 \mathbf{T}_{11}^\mathbf{\delta\mathbf{E}^\top}& \mathbf{T}_{12}^{\delta\mathbf{E}^\top} \\
-\mathbf{C}  & \mathbf{T}_{22}^{\delta\mathbf{E}^\top}
\end{bmatrix}
\end{align}
where the notation is consistent with the proof in Theorem~\ref{thm:our_unique}. So towards minimizing $\Vert \mathbf{Z}^\top\Vert_*$, we can always have $\mathbf{Z}^\top\mathbf{U}^\bot = \mathbf{0}$, or $\mathbf{T}_{12}^{\mathbf{Z}^\top}=\mathbf{T}_{22}^{\mathbf{Z}^\top} =\mathbf{0}$, for any $\delta\mathbf{E}^\top$. So the rest of spectrum of $\mathbf{Z}^\top$ is determined by $\mathbf{T}_{11}^\mathbf{Z^\top}$, and we have that $\mathbf{Z}^\top$ can have at most $r$ nonvanishing eigenvalues, where $r=\rank\left(\mathbf{X}\right)$. Note that $\left(\mathbf{T}_{11}^\mathbf{Z^\top}+\mathbf{T}_{11}^\mathbf{\delta\mathbf{E}^\top}\right)$ has $r$ eigenvalues of $1$, so spectrum of $\mathbf{T}_{11}^\mathbf{Z^\top}$ will be perturbation of that since the norm of $\mathbf{T}_{11}^\mathbf{\delta\mathbf{E}^\top}$ is in general small. This is also confirmed by our numerical experiments in \ref{sec:robust_spec}.

Moreover, we have intentionally left the norm for $\mathbf{E}$ unspecified since it apparently depends on the noise model we assume. The use of $\Vert\mathbf{E}\Vert_{2, 1}$ assumes the noise is sample-specific. In practice, however, a more natural assumption is uniformly random, i.e., each dimension of every data sample has the same chance of getting corrupted. In this case, the simple $\Vert \mathbf{E}\Vert_1$ will suffice. We demonstrate via experiments~\ref{sec:select}, and show that indeed $\Vert \cdot\Vert_1$ is more robust in that case.

The above comments about spectrum properties and noise model selection apply to both settings.

\subsection{Solving Robust LRR-PSD via Eigenvalue Thresholding}
The equivalence of \textbf{LRR-PSD} and \textbf{LRR} does not readily translate to the respective robust versions, and hence we need to figure out ways of solving the robust LRR-PSD. Due to the strong connection between these two problems, however, we will still try to employ the Augmented Lagrange Multipler (ALM) method (see e.g., \cite{lin2009augmented}) to tackle this as in ~\cite{liu2010robust}.

We first convert the problem into its equivalent form as
\begin{equation}
\min\limits_{\mathbf{Z}, \mathbf{E}, \mathbf{J}} \; \Vert \mathbf{J}\Vert_* + \lambda \Vert \mathbf{E} \Vert_{\ell}, \; \text{s.t.} \; \mathbf{X}=\mathbf{XZ}+\mathbf{E}, \mathbf{Z}=\mathbf{J}, \mathbf{Z}\succeq \mathbf{0},
\end{equation}
where we have used $\Vert \mathbf{E}\Vert_{\ell}$ to mean generic norms. Forming the partial ALM problem, we have
\begin{align}
\min\limits_{\mathbf{Z}, \mathbf{E}, \mathbf{J}\succeq \mathbf{0}, \mathbf{Y}_1, \mathbf{Y}_2} \;
& \Vert \mathbf{J}\Vert_* + \lambda \Vert \mathbf{E}\Vert_{\ell} \nonumber \\
 + & \langle \mathbf{Y}_1, \mathbf{X-XZ-E}\rangle + \langle \mathbf{Y}_2, \mathbf{Z-J}\rangle \nonumber \\
 + &\frac{\mu}{2}\Vert \mathbf{X-XZ-E}\Vert_F^2 + \frac{\mu}{2} \Vert \mathbf{Z-J}\Vert_F^2.
\end{align}
We can then follow the inexact ALM routine~\cite{lin2009augmented} to update $\mathbf{Z}$, $\mathbf{E}$, $\mathbf{J}$, $\mathbf{Y}_1$, $\mathbf{Y}_2$ alternately. While fixing others, how to update $\mathbf{E}$ depends on the norm $\Vert\cdot \Vert_{\ell}$. There are a bunch of norms that facilitate closed-form solutions, such as the $\Vert\cdot \Vert_{2, 1}$ discussed in \cite{liu2010robust} and $\Vert \cdot\Vert_1$ (see e.g., \cite{lin2009augmented}). How to update $\mathbf{J} (\mathbf{J}\succeq \mathbf{0})$ is the major obstacle to clean up. To be specific, we will be facing problem of this form to update $\mathbf{J}$
\begin{equation}
\mathbf{M}^* = \mathop{\arg \min}\limits_{\mathbf{M}} \; \frac{1}{\mu}\Vert \mathbf{M}\Vert_* + \frac{1}{2}\Vert \mathbf{M}-\mathbf{G}\Vert_F^2, \; \text{s.t.} \; \mathbf{M}\succeq 0,
\end{equation}
where $\mathbf{G}$ may or may not be symmetric. We will next show in Theorem~\ref{thm:sym} that symmetric $\mathbf{G}$ facilitates a closed-form solution, and generalize this in Theorem~\ref{thm:asymm} which basically states that asymmetric $\mathbf{G}$ also leads to a closed-form solution. Moreover, the major computational cost lies with eigen-decomposition of a symmetric square matrix, as compared with singular value decomposition of a square matrix of the same size in solving the counterpart in robust \textbf{LRR}.

\begin{lemma}[\cite{recht2008necessary}, Lemma 3.2] \label{lem:norm_ine}
For any block partitioned matrix
$
\mathbf{X} = \begin{bmatrix}
      \mathbf{A} & \mathbf{B} \\
      \mathbf{C} & \mathbf{D} \\
    \end{bmatrix},
$
this inequality holds
\begin{equation}
\Vert \mathbf{X}\Vert_* \geq
 \left\Vert \begin{bmatrix}
   \mathbf{A} & \mathbf{0} \\
   \mathbf{0} & \mathbf{D} \\
 \end{bmatrix}
 \right\Vert_* =
 \Vert \mathbf{A} \Vert_* + \Vert \mathbf{D} \Vert_*.
\end{equation}
Similar inequality also holds for the square of Frobenius norm $\Vert\cdot \Vert_F^2$.
 \end{lemma}

\begin{theorem} \label{thm:sym}
For any symmetric matrix $\mathbf{S}\in \mathcal{S}^n$, the unique closed form solution to the optimization problem
\begin{equation} \label{eq:opt_gen_diag}
\mathbf{M}^* = \mathop{\arg \min}\limits_{\mathbf{M}} \; \frac{1}{\mu}\Vert \mathbf{M}\Vert_* + \frac{1}{2}\Vert \mathbf{M}-\mathbf{S}\Vert_F^2,  \mathbf{M}\succeq 0,
\end{equation} takes the form
\begin{equation} \label{eq:opt_symm}
\mathbf{M}^* = \mathbf{Q}\; \Diag\left[\max(\bm{\lambda}-1/\mu, 0)\right] \mathbf{Q}^\top,
\end{equation}
whereby $\mathbf{S}=\mathbf{Q}\bm{\Lambda} \mathbf{Q}^\top$, for $\bm\Lambda=\Diag\left(\bm{\lambda}\right)$, is the spectrum(eigen-) decomposition of $\mathbf{S}$ and $\max\left(\cdot, \cdot\right)$ should be understood element-wise. \footnote{Toh and Yun~\cite{toh2009accelerated} have shed some light on the results (ref. Remark $3$ in their paper) but lack a detailed development and theoretic proof, and our proof is derived independent of their work. Moreover, solution to the general case as stated in the next theorem extends this results.}.
\end{theorem}
\begin{proof}
Observing that the objective is strictly convex over a convex set, we assert there exists a unique minimizer. The remaining task to single out the minimizer. Symmetric $\mathbf{S}$ admits a spectrum factorization $\mathbf{S}=\mathbf{Q}\bm\Lambda \mathbf{Q}^\top$, where $\mathbf{Q}^{-1} =\mathbf{Q}^\top$. We set $\widetilde{\mathbf{M}} = \mathbf{Q}^\top \mathbf{MQ}$, and hence the optimization in Eq.~\eqref{eq:opt_gen_diag} can be cast into
\begin{equation} \label{eq:opt_diag}
\widetilde{\mathbf{M}}^* = \mathop{\arg \min}\limits_{\widetilde{\mathbf{M}}} \; \frac{1}{\mu}\Vert \widetilde{\mathbf{M}}\Vert_* + \frac{1}{2}\Vert \widetilde{\mathbf{M}} - \bm\Lambda\Vert_F^2, \; \; \widetilde{\mathbf{M}}\succeq \mathbf{0}.
\end{equation}
By the unitary invariance property of the Frobenius norm and the nuclear norm, and the fact that $\mathbf{M}\succeq \mathbf{0} \Leftrightarrow \mathbf{Q}^\top \mathbf{M} \mathbf{Q}\succeq \mathbf{0}$ with unitary (orthogonal) $\mathbf{Q}$, we assert these two optimization problems are exactly equivalent (in the sense that $\mathbf{M}$ and $\widetilde{\mathbf{M}}$ can be recovered from each other deterministically).

Next we argue that a minimizer $\widetilde{\mathbf{M}}^*$ must be a diagonal matrix. Let $f(\widetilde{\mathbf{M}}) =1/\mu \Vert \widetilde{\mathbf{M}} \Vert_* + 1/2\Vert \widetilde{\mathbf{M}}-\bm\Lambda \Vert_F^2$. In fact, for a non-diagonal matrix $\widetilde{\mathbf{M}}_0$, we can always restrict it to diagonal elements to get $\widetilde{\mathbf{M}}_d$ such that $f(\widetilde{\mathbf{M}}_d)< f(\widetilde{\mathbf{M}}_0)$ by Lemma.~\ref{lem:norm_ine} and the fact $\bm\Lambda$ being diagonal. The strict inequality holds since restriction from a non-diagonal matrix to its diagonal elements results in strict decrease in square of the Frobenius norm. So assuming $\widetilde{\mathbf{M}}=\Diag \left(\xi_1, \cdots, \xi_n\right)$ and $\bm\Lambda=\Diag\left(\lambda_1, \cdots, \lambda_n\right)$, the problem reduces to a quadratic program w.r.t. $\left\{\xi_i\right\}_{i=1}^n$
\begin{equation}
\left\{\xi_i^*\right\}_{i=1}^n = \mathop{\arg \min}\limits_{\left\{\xi_i\right\}_{i=1}^n} \frac{1}{\mu} \sum_{i=1}^n \xi_i + \frac{1}{2}\sum_{i=1}^n \Vert \xi_i -\lambda_i \Vert^2, \;  \; \xi_i\geq 0, \forall{i}.
\end{equation}
The programming is obviously separable and simple manipulation suggests the unique closed form solution $\xi_i^* = \max(\lambda_i - 1/\mu, 0)$, which concludes the proof.
\end{proof}
\begin{remark}
Note that uniqueness of the solution may not be directly translated from Eq.~\eqref{eq:opt_diag} to \eqref{eq:opt_gen_diag} since one may argue $\mathbf{Q}$ is not unique in general. There are three causes to the ambiguity: 1) general sign reversal ambiguity of eigenvectors, 2) freedom with eigenvectors corresponding to the zero eigenvalues, and 3) freedom with eigenvectors corresponding to eigenvalues with multiplicity greater than $1$. Noticing that $\mathbf{M}^* =\sum_{i=1}^r \max (\lambda_i-1/\mu, 0)\mathbf{q}_i\mathbf{q}_i^\top$, the sign ambiguity and problems caused by zero-valued eigenvalues are readily removed in view of the form of the summand $\max (\lambda_i-1/\mu, 0)\mathbf{q}_i\mathbf{q}_i^\top$. For the last problem, assume one repeated eigenvalue $\lambda_i$ has one set of its eigenvectors arranged column-wise in $\mathbf{V} =[\mathbf{v}_1, \cdots, \mathbf{v}_k]$, which essentially spans a $k$-dimensional subspace (and acts as the basis). So this part of contribution to $\mathbf{M}^*$ can be written as $\max\left(\lambda_i-1/\mu, 0\right) \mathbf{VV}^\top$. Realizing that generating a new set of eigenvectors via linear combination can be accounted for by a rotation to the original basis vectors, namely $\widetilde{\mathbf{V}}=\mathbf{V}\mathbf{R}_{k\times k}$ for $\mathbf{R}^\top \mathbf{R} =\mathbf{I}$ in that subspace, we have  $\lambda_i \widetilde{\mathbf{V}}\widetilde{\mathbf{V}}^\top=\lambda_i \mathbf{VR}(\mathbf{VR})^\top =\lambda_i \mathbf{VV}^\top$. Hence the sum is not altered by any cause.
\end{remark}
In fact, building on Theorem~\ref{thm:sym}, we can proceed to devise a more general result on any real square matrix as follows.\footnote{Moreover, using the similar arguments, plus Lemma~\ref{lemma:norm-red}, we are able to produce a nonconstructive proof to the well known results about singular value thresholding~\cite{cai2008singular} without any use of subgradient. We will not pursue in this direction as it is out of the scope of this paper.}
\begin{theorem} \label{thm:asymm}
For any square matrix $\mathbf{P}\in \mathbb{R}^{n\times n}$, the unique closed form solution to the optimization problem
\begin{equation} \label{eq:opt_asym}
\mathbf{M}^* = \mathop{\arg \min}\limits_\mathbf{M} \; \frac{1}{\mu}\Vert \mathbf{M}\Vert_* + \frac{1}{2}\Vert \mathbf{M} - \mathbf{P}\Vert_F^2, \;  \; \mathbf{M}\succeq \mathbf{0},
\end{equation}
takes the form
\begin{equation}
\mathbf{M}^* = \mathbf{Q}\; \Diag\left[\max(\bm{\lambda}-1/\mu, 0)\right] \mathbf{Q}^\top,
\end{equation}
whereby $\widetilde{\mathbf{P}}=\mathbf{Q}\bm{\Lambda} \mathbf{Q}^\top$, for $\bm\Lambda=\Diag\left(\bm{\lambda}\right)$, is the spectrum(eigen-) decomposition of $\widetilde{\mathbf{P}}=\left(\mathbf{P}+\mathbf{P}^\top\right)/2$ and $\max\left(\cdot, \cdot\right)$ should be understood element-wise.
\end{theorem}
\begin{proof}
See Appendix~\ref{app:proof_assym}.
\end{proof}
The above two theorems (Theorem~\ref{thm:sym} and \ref{thm:asymm}) have enabled a fast solution to updating $\mathbf{J}$. Moreover, since they ensure
the symmetry of output $\mathbf{J}$ irrespective of the symmetry of $\mathbf{G}$, we can be assured the alternation optimization process converges to a solution of $\mathbf{J}$ that satisfies the constraint $\mathbf{J}\succeq \mathbf{0}$.
\section{Complexity Analysis and Scalability}\label{sec:imp}
For solving the ALM problems corresponding to robust LRR-PSD and robust LRR, the main computational cost per iteration comes from either eigen-decomposition of a symmetric matrix or SVD of a square matrix of the same size. In numerical linear algebra~\cite{golub1996matrix}, computing a stable SVD of matrix $\mathbf{X}\in\mathbb{R}^{n\times n}$ is to convert it to an symmetric eigen-decomposition problem on an augmented matrix
\[
\widetilde{\mathbf{X}} = \begin{bmatrix}
\mathbf{0}  & \mathbf{X}^\top  \\
\mathbf{X}  & \mathbf{0}
\end{bmatrix}.
\]
Hence from SVD to eigen-decomposition of comparable size, we can expect a constant factor of speedup that depends on the matrix dimension.
\begin{figure}
\centering
\includegraphics[width=0.35\textwidth]{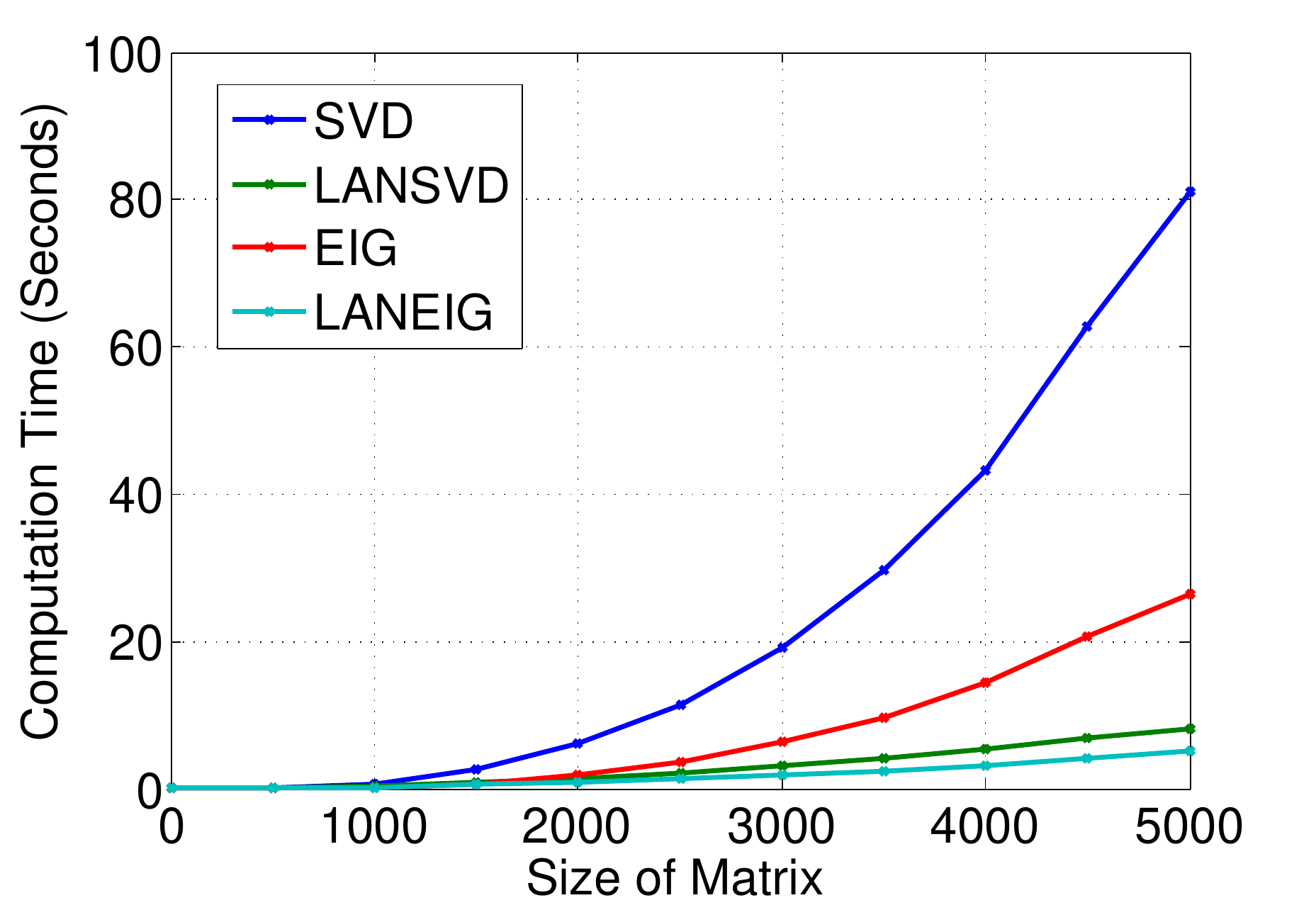} \vspace{-10pt}
\caption{\textbf{Comparison of computation time for full SVD/eigen- decomposition}. SVD and EIG are from Matlab built-in function (which essentially is wrapper for corresponding Lapack routines), and LANSVD, LANEIG from PROPACK. } \vspace{-10pt}
\label{fig:svd_comp}
\end{figure}
Figure~\ref{fig:svd_comp} provides benchmark results on computational times of SVD and eigen-decomposition on matrices of sizes ranging from very small up to $5000$. Tested solvers include these provided in Matlab and these in PROPACK~\cite{larsen-propack}. It is evident that for matrices of the same size, eigen-decomposition is significantly faster in both solver package. We will stick to the built-in functions in Matlab as we find in practice PROPACK is sometimes unstable when solving full problems (it is specialized in solving large and sparse matrices).

\section{Experiments} \label{sec:exp}
In this section, we systematically verify both the theoretic analysis provided before, and related claims.
\subsection{Experiment Setups}
We use two data sets throughout our experiments.
\\
\noindent \textbf{Toy Data} (TD). Following setting in~\cite{liu2010robust}, $5$ independent subspaces $\left\{\mathcal{S}_i\right\}_{i=1}^5\subset \mathbb{R}^{100}$ are constructed, whose bases $\left\{\mathbf{U}_i\right\}_{i=1}^5$ are generated by $\mathbf{U}_{i+1}=\mathbf{T}\mathbf{U}_i$, $1\leq i \leq 4$, where $\mathbf{T}$ represents a random rotation and $\mathbf{U}_1$ a random orthogonal matrix of dimension $100\times 4$. So each subspace has a dimension of 4. $20$ data vectors are sampled from each subspace by $\mathbf{X}_i=\mathbf{U}_i\mathbf{Q}_i$, $1\leq i \leq 5$ with $\mathbf{Q}_i$ being a $4\times 20$ iid zero mean unit variance Gaussian matrix $\mathcal{N}\left(0, 1\right)$. Collection of this clean $\mathbf{X}$ should have rank $20$.   \\
\noindent \textbf{Extended Yale B} (EYB). Following setting in~\cite{liu2010robust}, $640$ frontal face images of $10$ classes from the whole Yale B dataset are selected. Each class contains about $64$ images, and images are resized to $42\times 48$. Raw pixel values are stacked into data vectors of dimension $2016$ as features. This dataset is an example of heavily corrupted data.
\subsection{Equivalence of LRR-PSD and LRR}
\subsubsection{Spectrum Verification}
Recall the key to establish the equivalence of LRR-PSD and LRR lies with showing that the eigenvalues and singular values of $\mathbf{Z}^*$ are identical, with $1$ of multiplicity equal to the data rank and the rest 0's (Ref. Theorem~\ref{thm:our_unique} and the associated proof). In order to verify this, we use TD without introducing any noise, and hence the data matrix has rank $20$. We simulate the clean settings, i.e., LRR-PSD and LRR by gradually increasing the regularization parameter $\lambda$ of the robust versions \eqref{eq:robust-PSD} and \eqref{eq:robust-form}. Intuitively for large enough $\lambda$, the optimization tends to put $\mathbf{E}=\mathbf{0}$ and hence approaches the clean settings. Figure~\ref{fig:spec_clean} presents the results along the regularization path (0.1 $\sim$ 1). It is evident during the passing to $\lambda = 1$, the eigenvalue and singular value spectra match each other, and identically produce $20$ values of $1$ and the rest all $0$. This confirms empirically the correctness of our theoretic analysis.
\begin{figure}[!htbp]
  \centering
  \includegraphics[width=0.22\textwidth]{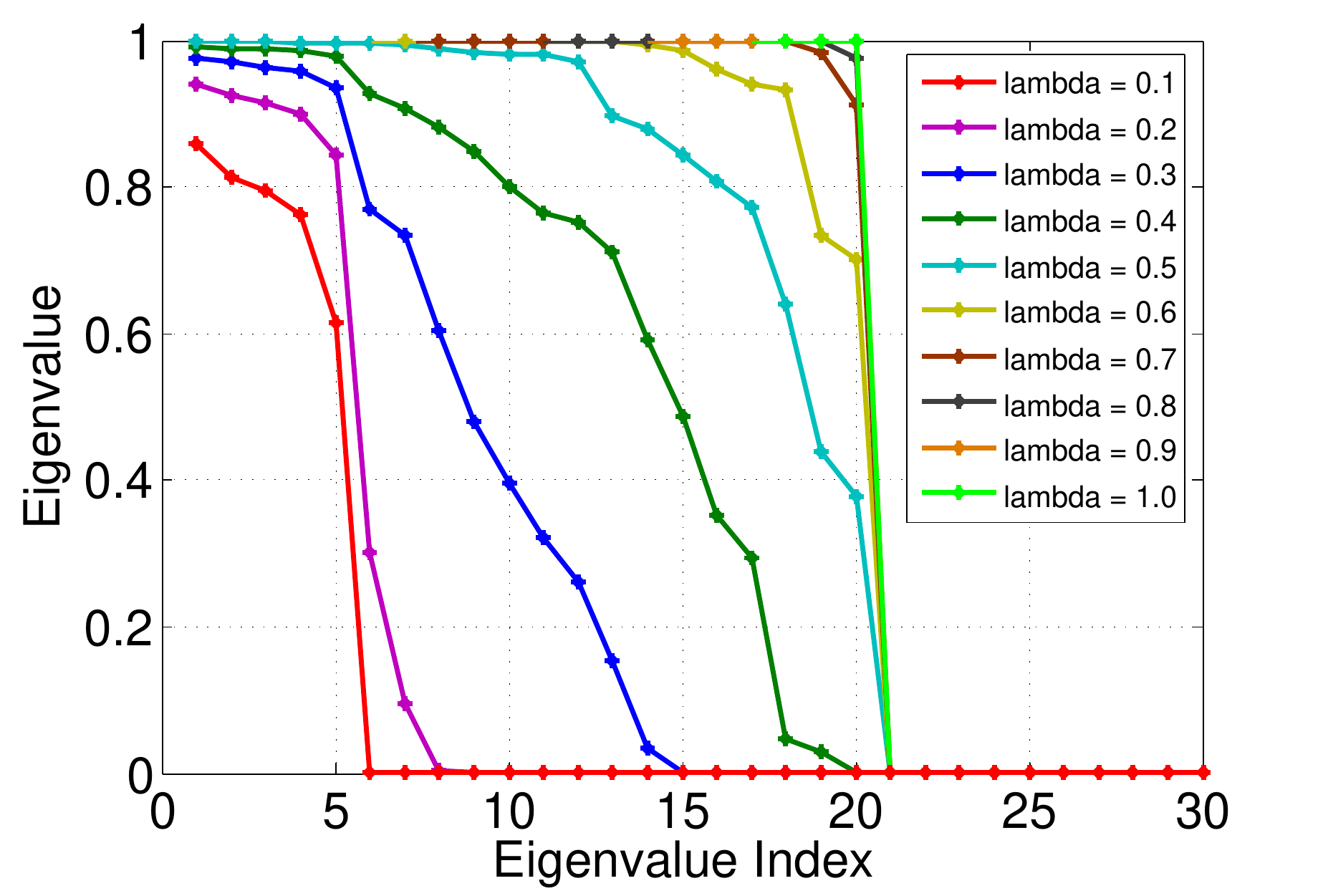}
  \includegraphics[width=0.22\textwidth]{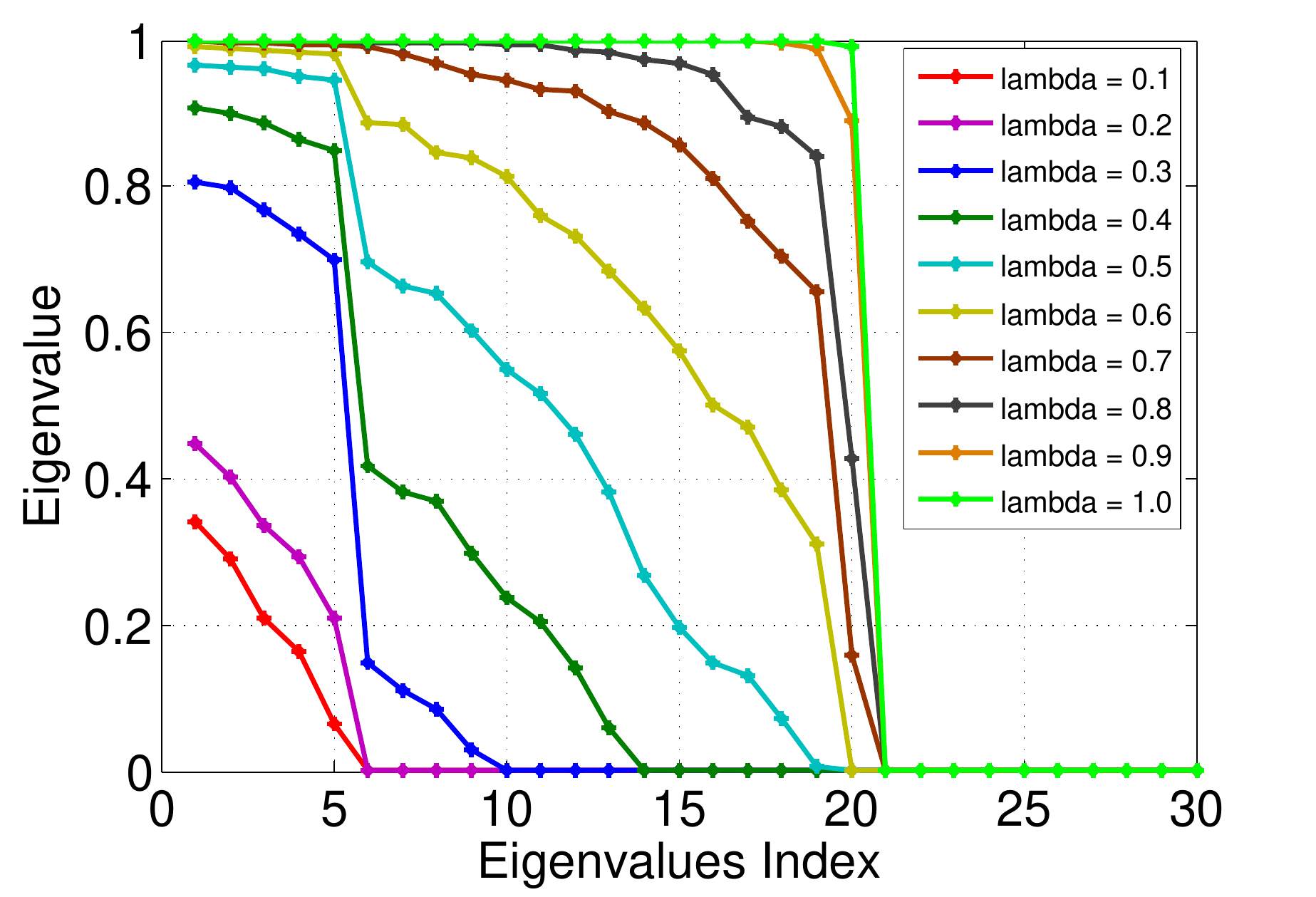}\\
  \includegraphics[width=0.22\textwidth]{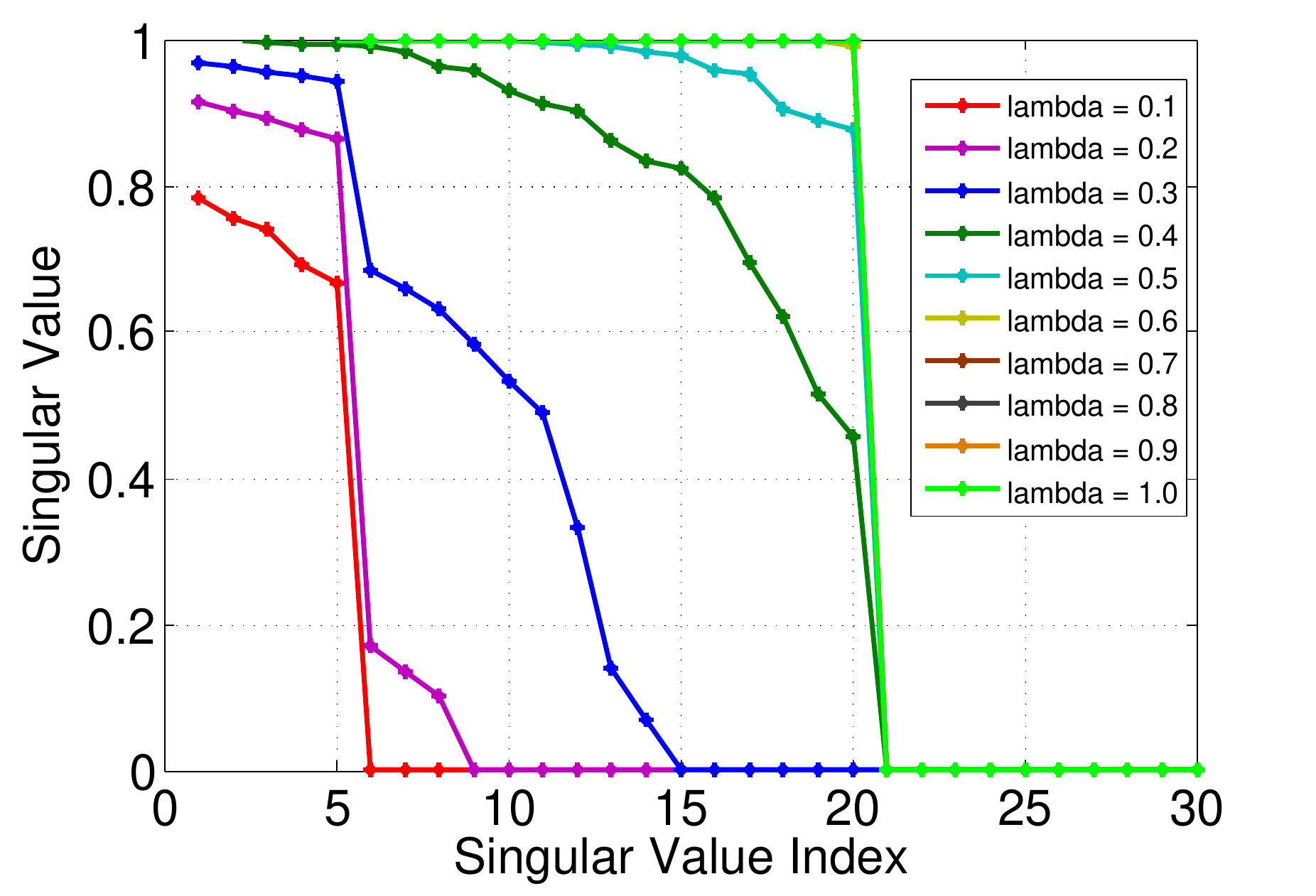}
  \includegraphics[width=0.22\textwidth]{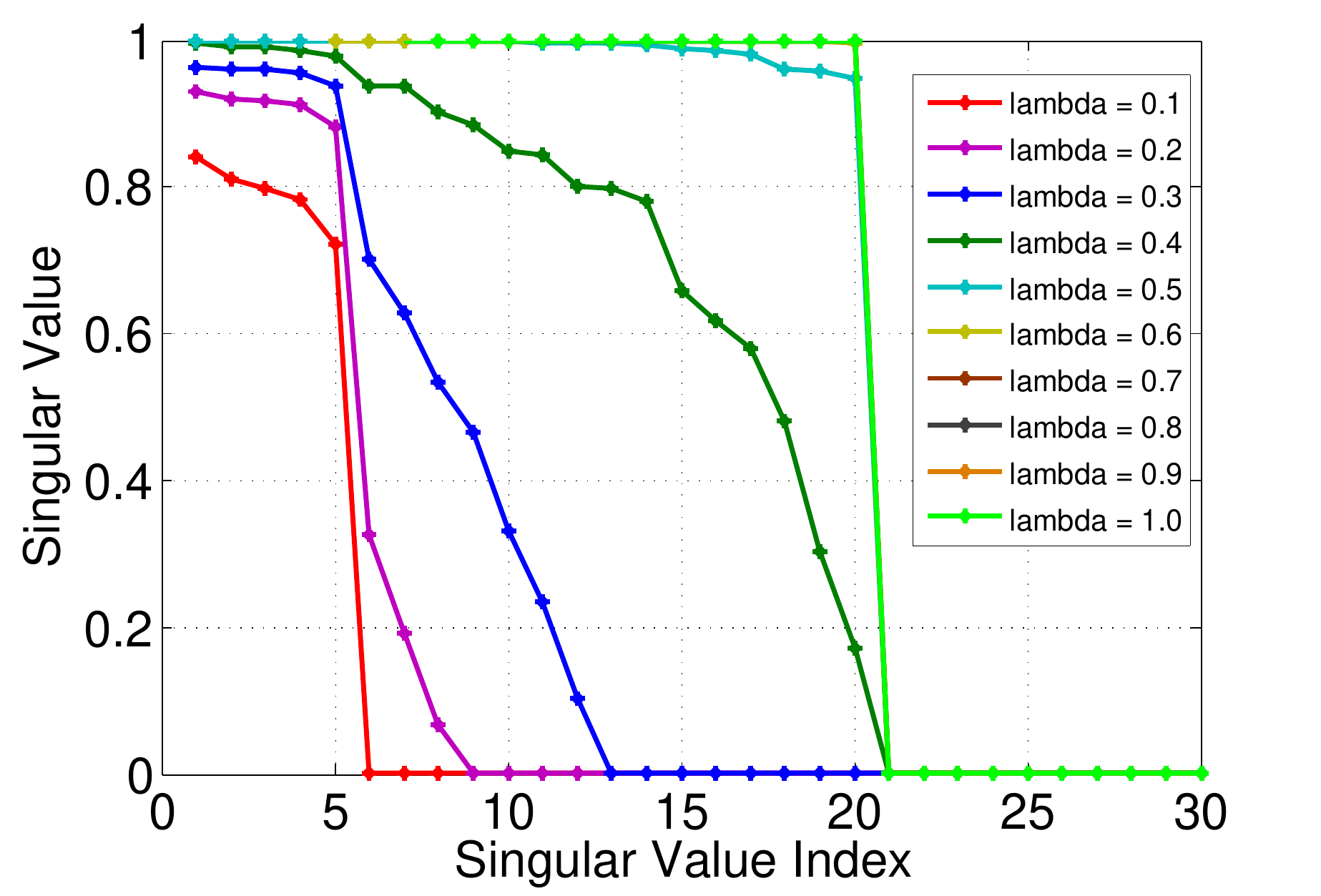} \vspace{-10pt}
  \caption{\textbf{Comparison of the eigen-spectrum (top) and singular value spectrum (bottom) for clean toy data (no artificial noises added) under the robust settings}. Increasing the value of $\lambda$ in the robust settings, or effectively passing towards the clean formulation, the optimal $\mathbf{Z}^*$ tends to produce $20$ nonvanishing eigenvalues/singular values of $1$. Left: by solving LRR. Right: by solving LRR-PSD. (Please refer to the color pdf and zoom in for better viewing effect.)} \vspace{-10pt}
  \label{fig:spec_clean}
\end{figure}
\subsubsection{Spectrum Perturbation Under Robust Setting} \label{sec:robust_spec}
As we conjectured in Sec.~\ref{sec:robust-set}, in most cases spectrum of the obtained affinity matrix from robust LRR-PSD or robust LRR will be perturbation of the ideal spectrum. Repeated experiments on many settings confirm about this, although we cannot offer a formal explanation to this yet. Here we only produce a visualization (Figure~\ref{fig:spec_pert}) to show how things evolve under different noise level when we set $\lambda=0.12$. The noise is added in sample-specific sense, as done in~\cite{liu2010robust}, i.e., some samples are chosen to be corrupted while others are kept clean. We do observe some breakdown cases when $\lambda$ is very small (not presented in the figure), which can be partially explained by that in that case the effect of nuclear norm regularization is weakened.
\begin{figure}[!htbp]
  \centering
  \includegraphics[width=0.22\textwidth]{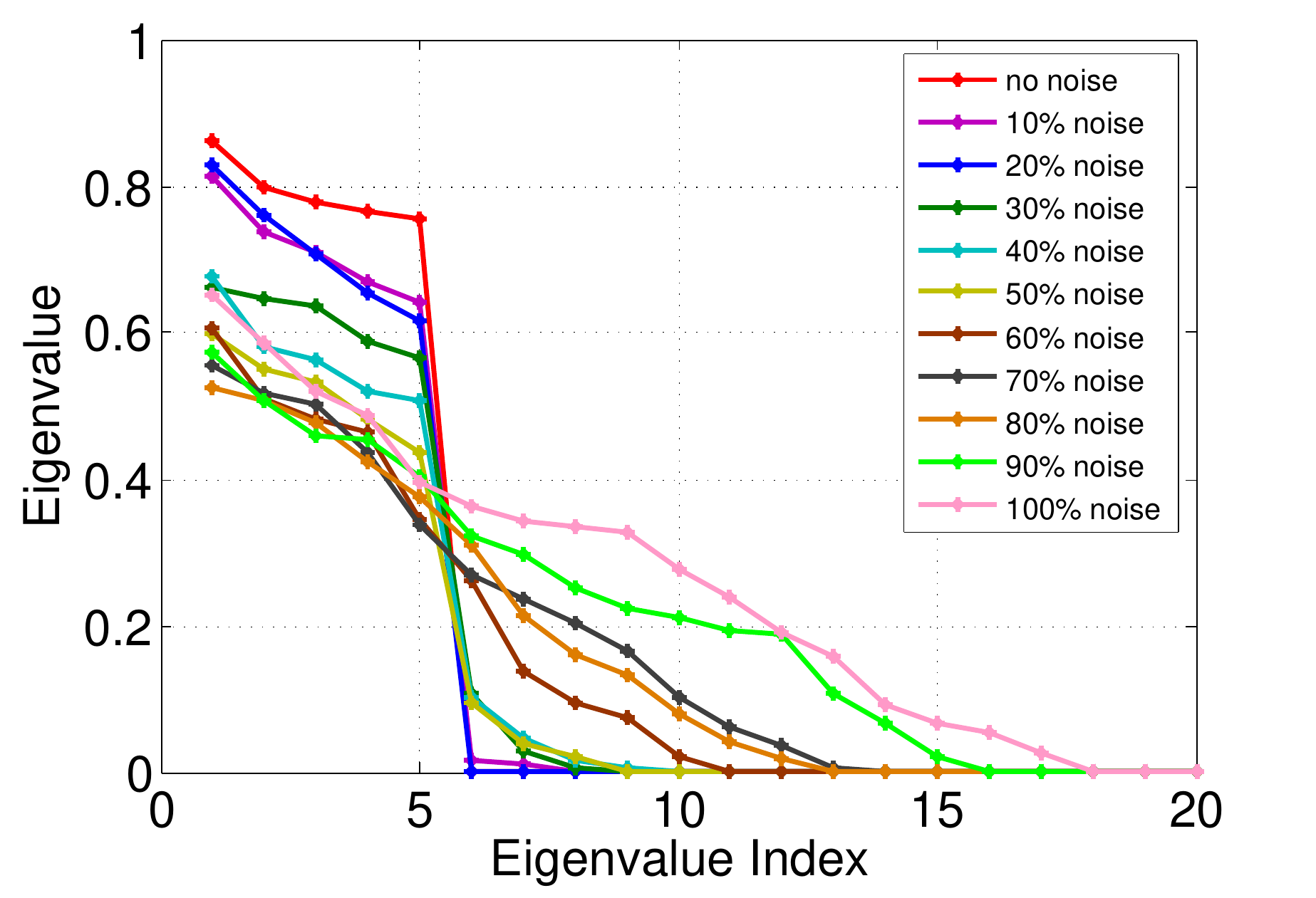}
  \includegraphics[width=0.22\textwidth]{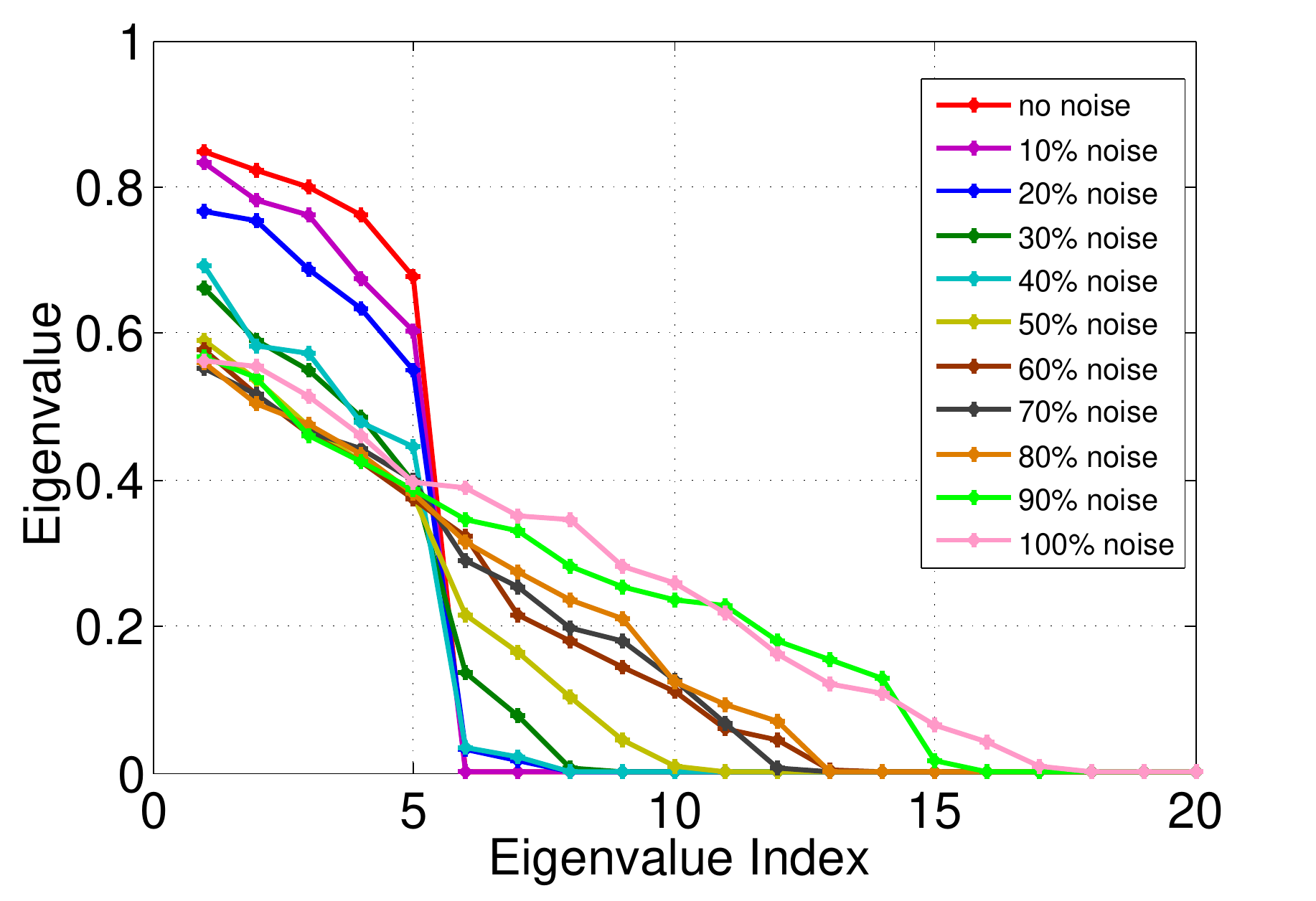} \vspace{-10pt}
  \caption{\textbf{Evolution of the eigen-spectrum of the learnt affinity matrix under different noise levels.} Left: solving by robust LRR; Right: solving by robust LRR-PSD. Surprisingly the spectra are always confined within $[0, 1]$ in this setting. (Please refer to the color pdf and zoom in for better viewing effect.)} \vspace{-10pt}
  \label{fig:spec_pert}
\end{figure}
\subsection{Selection of Noise Models} \label{sec:select}
We have argued that the norm selection for the noise term $\mathbf{E}$ should depend on the knowledge on noise patterns. We are going to compare the $\Vert \cdot\Vert_1$ noise model with the $\Vert \cdot\Vert_{2, 1}$ noise model used in~\cite{liu2010robust}.
\begin{figure}
  \centering
  \includegraphics[width=0.35\textwidth]{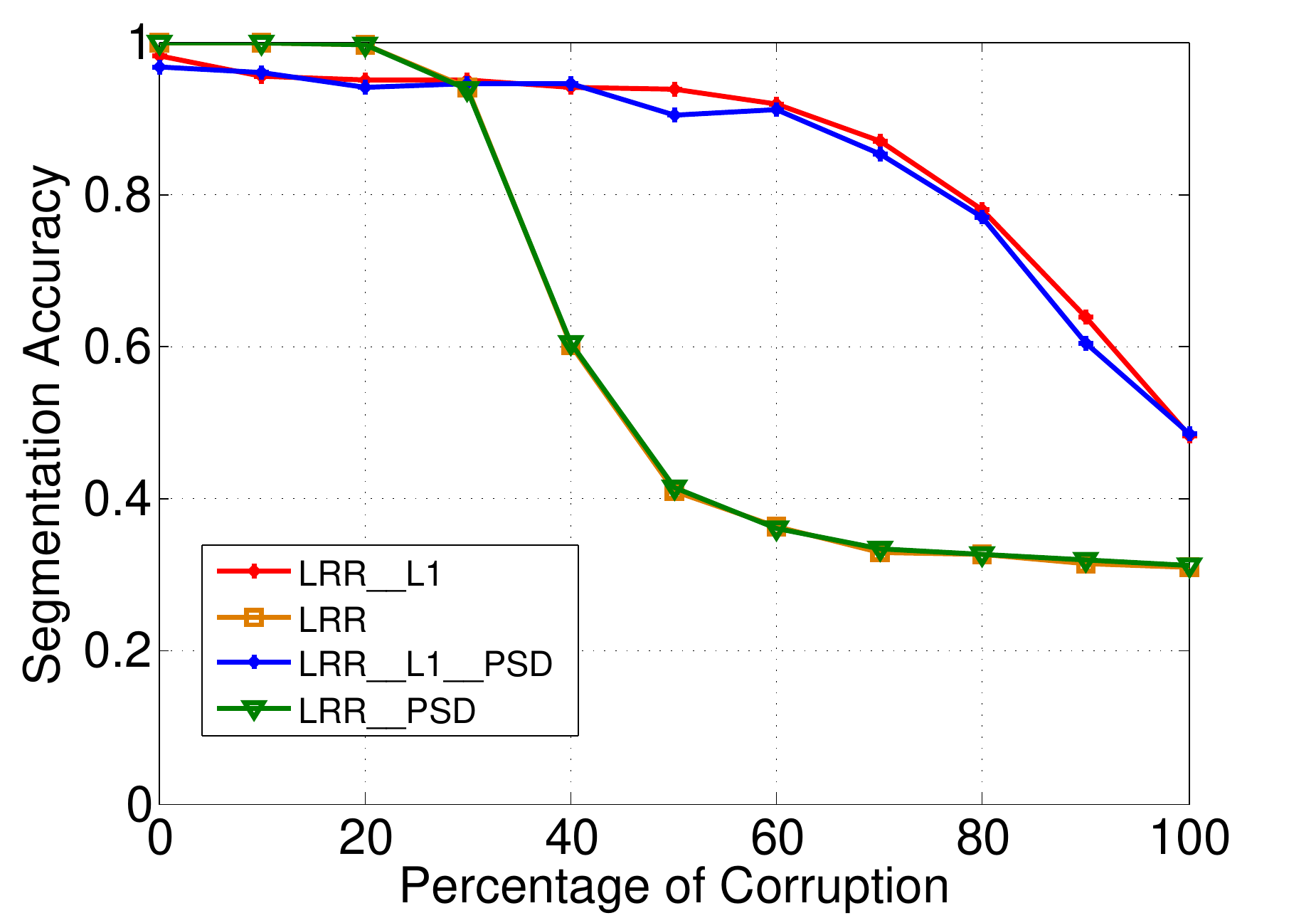} \vspace{-10pt}
  \caption{\textbf{Comparison of performance using different noise models.} In essence we use $\Vert\mathbf{E}\Vert_1$ and $\Vert \mathbf{E}\Vert_{2, 1}$ respectively in the objective. Instead of sample-specific noise, we assume random distributed noises, which is a more natural noise model. The $\ell_1$ version shows better resistance against noise. } \vspace{-10pt}
  \label{fig:norm_selection}
\end{figure}
First we test on TD. Instead of adopting a sample-specific noise assumption, we assume that the corruptions are totally random w.r.t. data dimension and data sample which is more realistic. We add Gaussian noise with zero mean and variance $0.3\Vert\mathbf{X}\Vert_F$, where $\mathbf{X}$ is the whole data collection. Percentage of corruption is measured against the total number of entries in $\mathbf{X}$. The evolution of SC performance against the percentage of corruption is presented in Figure~\ref{fig:norm_selection}. We can see the obvious better resistance against noises exhibited by the $\Vert\cdot \Vert_1$ form.
\subsection{Performance Benchmark: LRR-PSD vs. LRR}
We benchmark for the speed of robust LRR-PSD and robust LRR on EYB, and also present the clustering performance as compared to the conventional Gaussian kernel and linear kernel SC, which is obviously missing from~\cite{liu2010robust}.
\begin{table}[!htbp]
\caption{Segmentation accuracy ($\%$) on EYB. We record the average performance from multiple runs instead of the best, and reduce the dimension to 100 and 50 respectively in the bottom two rows.}
\centering
\vspace{2pt}
\begin{tabular}{c|c|c|c|c|c}
  \hline
    & Gauss SC & Linear SC & SSC  & LRR & LRR-PSD \\
    \hline
  Acc.          & 20.00     & 30.16 & 49.37 & \textbf{59.53} & \textbf{60.63} \\
  Acc. (100D)   & 22.66     & 27.97 & 49.38 & \textbf{61.56} & \textbf{60.00}       \\
  Acc. (50D)    & 24.84     & 27.97 & 49.22 & \textbf{62.83} & \textbf{61.81}    \\
  \hline
\end{tabular}
\label{tab:eyb_acc}
\end{table}
Table~\ref{tab:eyb_acc} presents the accuracy obtained via various affinity matrices for SC, with different setting of PCA pre-processing for noise removal\footnote{For SSC we used the implementation provided by the authors of~\cite{elhamifar2009sparse} with proper modification to their PCA routine.}. By comparison, obviously LRR-PSD and LRR win out and they are relatively robust against the PCA step, partially by virtue of their design to perform corruption removal together with affinity learning. To test the running time, we also include another set where each image in EYB is resized into $21\times 24$ (Set 1). We denote the original setting Set 2, and use the first $20$ classes of which each image resized into $42\times 48$ to produce Set 3. We report the running time (T), number of iterations (Iter), convergence tolerance (Tol) for each setting. Table~\ref{tab:eyb_sp} presents the results. Interestingly, LRR-DSP always converges with
\begin{table}[!htbp]
\caption{Running time and iterations on EYB. Advantage of LRR-PSD becomes significant as the data scale grows up.}
\centering
\vspace{2pt}
\begin{tabular}{c|c|c|c}
  \hline
  LRR/LRR-PSD & T (sec) & Iter & Tol \\
  \hline
  Set 1 & 271.87/\textbf{218.27} & \textbf{178}/330 & $10^{-6}$/$10^{-6}$ \\
  Set 2 & 475.23/\textbf{461.22} & \textbf{193}/496 & $10^{-6}$/$10^{-6}$ \\
  Set 3 & 3801.43/\textbf{2735.48} & \textbf{185}/392 & $10^{-6}$/$10^{-6}$ \\
  \hline
\end{tabular}
\label{tab:eyb_sp}
\end{table}
more iterations but less running time than that of LRR. The benefit of using eigen-decomposition in place of SVD is apparent.
\section{Summary and Outlook} \label{sec:conc}
In pursuit of providing more insights into recent line of research work on sparse-reconstruction based affinity matrix learning for subspace segmentation, we have discovered an important equivalence between the recently proposed LRR and our advocated version LRR-PSD in their canonical forms. This is a critical step towards understanding the behaviors of this family of algorithms. Moreover, we show that our advocated version, in its robust/denoising form, also facilitates a simple solution scheme that is as least as simple as the original optimization of LRR. Our experiments suggest in practice LRR-PSD is more likely to be flexible in solving large-scale problems.

Our current work is far from conclusive. In fact, there are several significant problems remained to be solved. First of all we observed in experiments the robust versions most of the times also produce affinity matrices with only positive eigenvalues, and themselves are very close to symmetric. We have not figured out ways to formally explain or even prove this. Furthermore, similar to the RPCA problem, it is urgent to provide theoretic analysis of the operational conditions of such formulation. From the computational side, SVD or eigen-decomposition on large matrices would finally become prohibitive. It would be useful to figure out ways to speed up nuclear norm optimization problems for practical purposes.

\section*{Acknowledgements}
This work is partially supported by project grant NRF2007IDM-IDM002-069 on ``Life Spaces" from the IDM Project Office, Media Development Authority of Singapore. We thank Prof. Kim-Chuan Toh, Mathematics Department of the National University of Singapore, for his helpful comments and suggestions to revision of the manuscript.

\appendix
\section{Appendix}
\subsection{Proof of Lemma~\ref{lem:norm_ine}}
\begin{proof}
Recall the fact that nuclear norm is dual to the spectral norm $\Vert\cdot\Vert_2$ (Lemma~\ref{lemma:duality}), the dual description follows
\begin{equation} \label{eq:nuclear_norm_super}
 \Vert \mathbf{X}\Vert_*
= \sup \left\{\left\langle \begin{bmatrix}
                                       \mathbf{Z}_{11} & \mathbf{Z}_{12} \\
                                       \mathbf{Z}_{21} & \mathbf{Z}_{22} \\
                                     \end{bmatrix},
                                      \begin{bmatrix}
                                       \mathbf{A} & \mathbf{B} \\
                                       \mathbf{C} & \mathbf{D} \\
                                     \end{bmatrix} \right\rangle\;\left\vert\;
                                     \left\Vert
                                     \begin{bmatrix}
                                       \mathbf{Z}_{11} & \mathbf{Z}_{12} \\
                                       \mathbf{Z}_{21} & \mathbf{Z}_{22} \\
                                     \end{bmatrix}
                                     \right\Vert_2 =1\right.
\right\},
\end{equation}
and similarly we also have
\begin{equation} \label{eq:nuclear_norm_sub}
\begin{aligned}
  & \left\Vert \begin{bmatrix}
   \mathbf{A} & \mathbf{0} \\
   \mathbf{0} & \mathbf{D} \\
 \end{bmatrix}
 \right\Vert_* \\
& =\sup
 \left\{\left\langle \begin{bmatrix}
                                       \mathbf{Z}_{11} & \mathbf{Z}_{12} \\
                                       \mathbf{Z}_{21} & \mathbf{Z}_{22} \\
                                     \end{bmatrix},
                                      \begin{bmatrix}
                                       \mathbf{A} & \mathbf{0} \\
                                       \mathbf{0} & \mathbf{D} \\
                                     \end{bmatrix} \right\rangle\;\left\vert\;
                                     \left\Vert
                                     \begin{bmatrix}
                                       \mathbf{Z}_{11} & \mathbf{Z}_{12} \\
                                       \mathbf{Z}_{21} & \mathbf{Z}_{22} \\
                                     \end{bmatrix}
                                     \right\Vert_2 =1\right.
\right\}  \\
& =\sup
\left\{\left\langle \begin{bmatrix}
                                       \mathbf{Z}_{11} & \mathbf{0} \\
                                       \mathbf{0} & \mathbf{Z}_{22} \\
                                     \end{bmatrix},
                                      \begin{bmatrix}
                                       \mathbf{A} & \mathbf{B} \\
                                       \mathbf{C} & \mathbf{D} \\
                                     \end{bmatrix} \right\rangle\;\left\vert\;
                                     \left\Vert
                                     \begin{bmatrix}
                                       \mathbf{Z}_{11} & \mathbf{0} \\
                                       \mathbf{0} & \mathbf{Z}_{22} \\
                                     \end{bmatrix}
                                     \right\Vert_2 =1\right.
\right\} \\
& = \Vert \mathbf{A}\Vert_* + \Vert \mathbf{D}\Vert_*.
\end{aligned}
\end{equation}
Since \eqref{eq:nuclear_norm_sub} is a supremum over a subset of that in \eqref{eq:nuclear_norm_super}, the inequality about the nuclear norm holds. The claim about the square of Frobenuis norm holds trivially from nonnegativeness of any block norm squares contributed to the total norm square.
\end{proof}
\subsection{Proof of Theorem~\ref{thm:asymm}} \label{app:proof_assym}
\begin{proof}
Similarly the program is strictly convex and we expect a unique minimizer. By the semi-definiteness constraint, we are only interested in $\mathbf{M}\in \mathcal{S}^n$. Hence $\Vert \mathbf{M} - \mathbf{P}\Vert_F^2 = \Vert \mathbf{M} - \mathbf{P}^\top\Vert_F^2$, which suggests the objective function can be cast in its equivalent form $1/\mu\Vert \mathbf{M} \Vert_*+$$1/4\Vert \mathbf{M-P}\Vert_F^2+$$1/4\Vert \mathbf{M-P}^\top\Vert_F^2$. Further we observe that
\begin{equation}
 \Vert \mathbf{M-P}\Vert_F^2 + \Vert \mathbf{M}-\mathbf{P}^\top\Vert_F^2 = \Vert \mathbf{M}- (\mathbf{P+P}^\top)/2\Vert_F^2 + \mathcal{C}\left(\mathbf{P}\right)
\end{equation}
where $\mathcal{C}(\mathbf{P})$
only depends on $\mathbf{P}$ (are hence constants)
. Hence we reach an equivalent formation of the original program Eq.~\eqref{eq:opt_asym} as
\begin{equation} \label{eq:opt_asym_ref}
\mathbf{M}^* = \mathop{\arg \min}\limits_\mathbf{M} \; \frac{1}{\mu}\Vert \mathbf{M}\Vert_* + \frac{1}{2}\Vert \mathbf{M} - \widetilde{\mathbf{P}}\Vert_F^2, \; \text{s.t.} \; \mathbf{M}\succeq \mathbf{0},
\end{equation}
with $\widetilde{\mathbf{P}}=(\mathbf{P+P}^\top)/2$. Solution to Eq.~\eqref{eq:opt_asym_ref} readily follows from Theorem~\ref{thm:sym}.
\end{proof}

\bibliographystyle{IEEEtran}
\bibliography{SSC}
\end{document}